\documentclass[twocolumn, journal]{IEEEtran}
\usepackage{ifpdf}
\usepackage{cite}
\usepackage{bbm}
\ifCLASSINFOpdf
\usepackage{graphicx}
\usepackage{subfigure}
\usepackage{epstopdf}
\usepackage{xcolor}
\usepackage{kotex}
\else
\fi
\usepackage{float}
\usepackage{amsmath}
\usepackage{amssymb}
\usepackage{amsthm}
\usepackage{mathtools}

\newtheorem{lemma}{Lemma}

\newtheorem{theorem}{Theorem}
\newtheorem{remark}{Remark}

\newtheorem{assumption}{Assumption}
\usepackage{algorithm}

\usepackage{algpseudocode}

\allowdisplaybreaks
\usepackage{array}

\usepackage{stfloats}
\usepackage{tikz}

\usepackage[shortlabels]{enumitem}
\usepackage{comment}
\usepackage{multirow}
\usepackage{cleveref}
\usepackage{booktabs}
\renewcommand{\arraystretch}{1.4}
\usepackage{pifont}
\newcommand{\cmark}{\ding{51}}
\newcommand{\xmark}{\ding{55}}
\usepackage{array}
\usepackage{makecell}

\begin{document}
\title{Importance-Aware Feature Sparsification for Wireless Split Learning}
\author{
    \IEEEauthorblockN{Bumjun Kim},~\IEEEmembership{Graduate Student Member,~IEEE},
    \IEEEauthorblockN{Yoon Huh},~\IEEEmembership{Graduate Student Member,~IEEE},
	and
	\IEEEauthorblockN{Wan Choi},~\IEEEmembership{Fellow,~IEEE}

    \thanks{A part of this paper \cite{kim2025feature} was presented in IEEE Vehicular Technology Conference (VTC) Fall, Chengdu, China, Oct. 19-23, 2025.}
    \thanks{Received 25 November 2025; revised 04 June 2026; accepted 29 September 2026. Date of publication xx October 2026; date of current version xx October 2026. This work was supported by Institute of Information \& communications Technology Planning \& Evaluation (IITP) grant funded by the Korea government(MSIT) (No. RS-2026-25523842). (\emph{Corresponding Author: Wan Choi})}
	\thanks{B.~Kim, Y.~Huh and W.~Choi are with the Department of Electrical and Computer Engineering, Seoul National University (SNU), and the Institute of New Media and Communications, SNU, Seoul 08826, Korea. (e-mail: \{eithank96, mnihy621, wanchoi\}@snu.ac.kr)}
        \vspace{-3mm}

}

\maketitle

\begin{abstract}
Wireless split learning (SL) reduces on-device computation by offloading upper layers to a server, yet transmitting high-dimensional intermediate features at each iteration remains a major communication bottleneck. Existing methods select features at the client side using task-agnostic criteria such as magnitude, statistics, or clustering, which increases client-side processing and often degrades accuracy under non-independent and identically distributed (non-i.i.d.) client data. We propose importance-aware class-balanced sparsification (ICS), a lightweight approach in which the server ranks feature channels using Grad-CAM-based scores obtained from the true-class logit during backpropagation. The per-class scores are aggregated into a class-balanced, label-agnostic importance vector that mitigates head-class bias under label skew, and each client reuses this vector in the next round to retain the top-$N$ feature channels, incurring no additional client-side forward or backward passes. We further derive a non-asymptotic convergence bound that isolates the sparsification-induced error and characterizes how the sparsification ratio and mini-batch size jointly affect convergence under a fixed communication budget, and we analyze the communication and computational overhead of ICS against representative baselines. Beyond sequential CNN-based SL, we extend ICS to parallel split learning and to transformer-based models. Experiments show that ICS consistently outperforms the baselines, with larger gains under severe non-i.i.d. partitions.
\end{abstract}

\section{Introduction}
Driven by the expansion of the internet of things (IoT) and advances in next generation wireless technologies, the number of edge devices that capture data in real time continues to grow at an extraordinary pace. These devices produce diverse and large volumes of information including sensor readings, user profiles, and video streams, which now power AI applications across healthcare, natural language processing, transportation, and other domains. The conventional centralized learning (CL) trains models at a server by collecting raw data from devices, yet this approach increasingly confronts hard limits because data must leave its source devices, which heightens privacy risk and collides with regulatory requirements \cite{2021chen}. Consequently, CL is often impractical for modern edge intelligence, motivating approaches that keep data local while still enabling collaborative training.

Federated learning (FL) \cite{kim2024privacy,  huh2025feature, abrar2025biased,lin2020federated} enables collaborative model training while keeping data local to each client. In FL, clients update their models using private data and transmit only parameter updates or gradients to a central server, thereby reducing the risk of sensitive information leakage. While this paradigm avoids direct exposure of raw data, it still requires fully on-device training, which can be prohibitive for resource-constrained clients. For example, ResNet-50 has roughly $25$M parameters and requires more than $3.8$ GFLOPs for a single forward pass \cite{he2016deep}; the backward pass doubles or even triples the computation, while also increasing memory usage due to activations and optimizer states. On devices such as smartphones, wearables, and IoT nodes, these requirements translate into higher latency, energy consumption, and thermal load, often becoming the practical bottleneck in FL deployments.

Split learning (SL) \cite{vepakomma2018split, 2024Lin, 2024Lin2, 2023Lan} addresses this bottleneck by redistributing computation between clients and the server. The network is partitioned at a designated cut layer into a lightweight client-side model and a heavier server-side model. During training, each client performs the forward pass only up to the cut layer, producing intermediate features that are transmitted to the server. The server then completes the forward pass, computes the loss, backpropagates through its portion of the model, and returns the gradients with respect to the intermediate features. Using these gradients, clients update their local parameters. This division shifts the most compute- and memory-intensive layers to the server, significantly reducing on-device FLOPs and activation footprints, thereby lowering energy consumption and end-to-end latency. In the vanilla SL protocol involving multiple clients, this interaction is typically performed sequentially, only one client communicates with the server at a time, forcing clients to take turns. This split topology is also maintained during the inference phase. The client computes and transmits the intermediate features, and the server performs the remaining forward pass to return the final result \cite{lee2023wireless, Universalyoon2025, huh2026markov}.

Despite these advantages, SL still suffers from high training
latency due to sequential client-server interactions \cite{jeon2020privacy,thapa2022splitfed} and per-iteration transmission of high-dimensional intermediate features \cite{zhou2024mask,oh2025communication,wang2022fedlite}. To alleviate the waiting overhead caused by sequential client participation, parallel split learning (PSL) methods \cite{jeon2020privacy,thapa2022splitfed} allow multiple clients to perform client-side computation and transmit intermediate features in parallel. However, such parallelism comes at the cost of supporting simultaneous uplink transmissions within the same training iteration, which requires larger aggregate uplink resources in wireless systems. Moreover, it does not reduce the intermediate feature overhead transmitted by each client. Since intermediate features must be uploaded at every iteration, their cumulative communication cost can dominate the overall training overhead, regardless of whether clients participate sequentially or in parallel.

To reduce the intermediate feature overhead, several communication-efficient SL methods were proposed. Existing studies typically decide which elements of the intermediate feature are retained for uplink transmission based on random sampling \cite{zhang2024federated}, magnitudes \cite{zhou2024mask}, feature statistics \cite{oh2025communication}, or clustering \cite{wang2022fedlite}. Although these methods reduce the number of transmitted feature elements, their selection strategies are generally computed on the client side. This partially contradicts the primary motivation of SL, which is to offload computationally demanding operations from resource-constrained clients to the server. Moreover, these strategies do not reflect the task-specific contribution of each feature to the final prediction, making it difficult to ensure that the transmitted features are task-relevant.

This issue becomes more critical under non-independent and identically distributed (non-i.i.d.) client data. In practical wireless learning scenarios, clients often collect data from diverse environments, resulting in heterogeneous data distributions. Under such data skew, feature selection based solely on client-side information can be biased toward locally dominant classes. As a result, the transmitted features may overrepresent client-specific characteristics rather than class-balanced task relevance, which can degrade the global training performance. Therefore, an effective sparsification strategy for wireless SL should be communication-efficient, task-aware, robust to non-i.i.d. data skew, and lightweight at the client side.

\begin{table*}
\centering
\caption{Comparison of FL, sequential SL and PSL.}
\label{tab:fl_sl_psl}
\renewcommand{\arraystretch}{1.25}

\begin{tabular}{p{0.11\linewidth}p{0.27\linewidth}p{0.24\linewidth}p{0.25\linewidth}}
\hline
\textbf{Paradigm} 
& \textbf{Main idea}
& \textbf{Client-side computation}
& \textbf{Uplink communication overhead} \\ \hline

FL
& Each client trains the full model locally, and the server aggregates the uploaded model updates.
& Full forward/backward pass over the entire model at each client, i.e., $\mathcal{O}(\mathcal{C}_{d}+\mathcal{C}_{s})$.
& Model updates are uploaded from the clients, i.e., $\mathcal{O}(K(P_{d}+P_{s}))$ in total. \\

Sequential SL
& The model is split into client-side and server-side parts, and clients communicate with the server sequentially.
& Forward/backward pass only over the client-side model at each client, i.e., $\mathcal{O}(\mathcal{C}_{d})$.
& Intermediate features are uploaded sequentially, i.e., $\mathcal{O}(KBCUV)$ in total. \\

PSL
& Multiple clients perform client-side computation and transmit intermediate features to the server in parallel. Client-side models are synchronized after the local updates.
& Forward/backward pass over the client-side model at each client, i.e., $\mathcal{O}(\mathcal{C}_{d})$.
& Intermediate features and client-side model updates are uploaded, i.e., $\mathcal{O}(K(BCUV+P_{d}))$ in total. \\ \hline

\end{tabular}

\end{table*}

Motivated by these observations, we propose an \textit{importance-aware class-balanced sparsification} \textbf{(ICS)}, inspired by gradient-weighted class activation mapping (Grad-CAM) \cite{selvaraju2017grad}, originally developed in the context of explainable AI (XAI). Grad-CAM measures the contribution of intermediate features to class-specific predictions, highlighting the most salient intermediate feature channels. Recently, Grad-CAM has been used to guide resource allocation in semantic communication systems. In particular, the authors in \cite{zhou2025fast} employed Grad-CAM-inspired feature importance scores on a pre-trained JSCC encoder to learn a feature importance predictor at the transmitter, which is then used to allocate space–time resources during inference. In contrast, our work integrates Grad-CAM directly into the training loop of wireless SL under distributed clients with non-i.i.d. data, transforming it from a post-hoc interpretability tool into a core component of the SL training process. The server computes all importance scores during the backpropagation stage, eliminating the need for any transmitter-side processing.

Specifically, the server computes class-specific importance vectors, and aggregates them into a \emph{class-balanced, label-agnostic} importance vector. This \emph{class-balanced} vector design ensures robustness to client label skew, while its \emph{label-agnostic} design enables reuse of the sparsification during inference, where ground-truth labels are unavailable. This importance vector is then returned to the client and reused in the subsequent iteration to sparsify intermediate features for uplink transmission, ensuring that only the most discriminative channels are transmitted within a given bandwidth budget. Crucially, this entire process is performed by the server, imposing no additional computation overhead on the resource-constrained clients. This design choice strictly adheres to the core philosophy of SL, which is to offload demanding computations away from the clients.

Furthermore, we provide a convergence analysis that addresses a critical aspect often overlooked in wireless SL. Unlike in FL, the uplink overhead in SL scales directly with the mini-batch size. Our analysis explicitly quantifies the error introduced by the sparsification ratio and characterizes its joint interaction with the batch size, revealing a fundamental trade-off between stochastic gradient variance and per-sample information fidelity under a fixed communication budget.

Beyond the sequential SL, we further extend \textbf{ICS} to PSL. Since multiple clients transmit intermediate features in parallel in the framework, the server constructs a shared class-balanced importance vector by aggregating importance scores across participating clients.

\subsection{Related Works}

\subsubsection{Distributed Learning Frameworks}
FL keeps raw data local but requires each client to train the full model and upload model updates, imposing substantial computation and memory burden on resource-constrained clients. In contrast, sequential SL partitions the model between the client and the server, so each client executes only the client-side model while exchanging intermediate features and activation gradients with the server. PSL further reduces sequential waiting time by allowing multiple clients to perform client-side computation and upload intermediate features in parallel, which requires larger aggregate uplink resources within the same training iteration. These differences highlight that SL-based frameworks reduce client-side computation but make intermediate feature transmission a central communication bottleneck. Table~\ref{tab:fl_sl_psl} summarizes the main differences among FL, sequential SL and PSL with the main notations defined in Table~\ref{tab:notation}.
\begin{table*}
\caption{Comparison of communication-efficient split learning methods.}
\label{tab:related_work_comparison}

\centering
\renewcommand{\arraystretch}{1.25}
\setlength{\tabcolsep}{3pt}
\begin{tabular}{
>{\raggedright\arraybackslash}p{0.22\textwidth}
>{\raggedright\arraybackslash}p{0.27\textwidth}
>{\centering\arraybackslash}p{0.14\textwidth}
>{\centering\arraybackslash}p{0.19\textwidth}
>{\centering\arraybackslash}p{0.12\textwidth}
}
\hline
\textbf{Method}
& \textbf{Selection criterion}
& \makecell{\textbf{Task-aware}\\\textbf{importance}}
& \makecell{\textbf{No extra client-side}\\\textbf{computation}}
& \makecell{\textbf{Non-i.i.d.}\\\textbf{consideration}} \\ \hline

RS~\cite{zhang2024federated}
& Random sampling
& \xmark
& \cmark
& \xmark \\

TS~\cite{zhou2024mask}
& Feature magnitude
& \xmark
& \xmark
& \xmark \\

RTS~\cite{zheng2023reducing}
& Feature magnitude + random sampling
& \xmark
& \xmark
& \xmark \\

FedLite~\cite{wang2022fedlite}
& Feature clustering
& \xmark
& \xmark
& \xmark \\

SplitFC~\cite{oh2025communication}
& Feature statistics
& \xmark
& \xmark
& \xmark \\

\textbf{Proposed ICS}
& \textbf{Grad-CAM-based channel importance}
& \cmark
& \cmark
& \cmark \\ \hline
\end{tabular}

\end{table*}

\subsubsection{Parallel Split Learning Methods}
PSL~\cite{jeon2020privacy, thapa2022splitfed} was developed to mitigate the waiting time caused by sequential client participation by transmitting intermediate features to the server in parallel. CPSL~\cite{wu2023split} extended this idea by grouping clients into multiple clusters, where clients are trained in parallel within each cluster and inter-cluster aggregation is performed sequentially to further reduce training latency. 

More recently, several studies  addressed system-level optimization issues in PSL frameworks. ESFL~\cite{zhu2024esfl} jointly optimized the model splitting point and server-side computing resource allocation under heterogeneous device resources. AdaptSFL~\cite{lin2025adaptsfl} adaptively controlled the model splitting point and client-side aggregation interval using convergence analysis. The authors in~\cite{wang2025split} further considered dynamic wireless environments by jointly optimizing model splitting, server-side computing resource allocation, and transmission power control. HSFL~\cite{lin2025hierarchical} extended to hierarchical systems through tier-wise model aggregation and model splitting optimization.
\subsubsection{Communication-Efficient Split Learning Methods}
Communication-efficient SL methods aim to reduce the transmission overhead caused by high-dimensional intermediate features. FedLite~\cite{wang2022fedlite} introduced a scalable SL framework for resource-constrained clients by clustering similar intermediate features and applying gradient correction to compressed features. SplitFC~\cite{oh2025communication} further reduced communication overhead through adaptive feature-wise compression, where intermediate feature vectors are sparsified according to standard-deviation-based probabilities. Randomized and magnitude-based sparsification methods~\cite{zhang2024federated,zheng2023reducing,zhou2024mask} were also  studied to reduce the intermediate feature overhead by selecting only a subset of feature elements for uplink transmission.

Although these methods reduce the number of transmitted feature elements, they still have important limitations. First, feature selection is generally performed at the client side, requiring additional operations such as feature clustering, feature-statistics computation, or magnitude-based ranking before transmission. This increases the preprocessing burden on resource-constrained clients. Second, these methods select features without explicitly considering their task-specific contribution to the final prediction. Therefore, the retained features may fail to preserve task-relevant information, especially under non-i.i.d. client data.

To further clarify the novelty of \textbf{ICS}, Table~\ref{tab:related_work_comparison} compares \textbf{ICS} with existing communication-efficient SL methods. Unlike other approaches, \textbf{ICS} evaluates feature importance using Grad-CAM-based class-specific saliency computed at the server. The resulting class-balanced importance vector allows clients to sparsify intermediate features without additional neural network computation at the client side, while preserving task-aware and non-i.i.d.-robust feature selection.

\subsection{Contributions}
The main contributions of this paper are as follows: 
\begin{itemize} 
\item \textbf{Importance-aware feature sparsification:} We propose a novel sparsification method, termed  \textbf{ICS}, which leverages Grad-CAM to evaluate the task-specific importance of intermediate features. By using importance vectors computed at the server in previous iterations as surrogates, the proposed approach enables efficient client-side feature compression without incurring additional forward–backward passes at the client side.
\item \textbf{Class-balanced label-agnostic importance vector:} We design a class-balanced aggregation mechanism that converts class-specific importance vectors into a unified label-agnostic importance vector. This design mitigates class bias under non-i.i.d. client data and preserves class-balanced task-relevant features, while also enabling feature sparsification during inference where ground-truth labels are unavailable.
\item \textbf{Theoretical analysis:} We provide a theoretical analysis of \textbf{ICS} in terms of both system overhead and convergence. For the overhead analysis, we analyze the uplink communication overhead and the computational complexity of \textbf{ICS} against conventional SL and existing communication-efficient SL methods.  The analysis  shows that \textbf{ICS} reduces the transmitted intermediate feature overhead without transmitting additional mask indices and shifts the importance-estimation cost from the clients to the server. For the convergence analysis, we derive a non-asymptotic bound that explicitly characterizes the sparsification error and reveals the trade-off between the mini-batch size and the sparsification ratio under a fixed uplink communication budget.
\item \textbf{Extension to PSL frameworks:} We extend \textbf{ICS} beyond sequential SL to PSL. In the setting, the server constructs a shared class-balanced importance vector by aggregating Grad-CAM-based importance scores across participating clients, enabling importance-aware sparsification under parallel client participation. 
\item \textbf{Extensive experimental validation:} We validate \textbf{ICS} across diverse tasks and datasets, showing consistent accuracy gains and faster convergence over baselines, especially under non-i.i.d. client data. In addition, we supplement these results with a qualitative Grad-CAM-based ablation study for the classification task, which visually confirms that \textbf{ICS} consistently identifies important, task-relevant features.
\end{itemize}

\subsection{Notations}
For clarity, the key notations used in the paper are summarized in Table \ref{tab:notation}.

\begin{table}
\caption{Notations and Descriptions.}
\label{tab:notation}

\centering
\renewcommand{\arraystretch}{1.3}
\begin{tabular}{ll}
\hline
\textbf{Symbol}                 & \textbf{Description}                                \\ \hline
$K$                             & Number of clients.                                  \\
$T$                             & Number of iterations.                               \\
$E$                             & Number of local steps.                              \\
$B$                             & Mini-batch size.                                    \\
$L$ & Number of output classes. 
                           \\
$\mathcal{L}$ & Per-sample loss function.  
                           \\
$C,\,U,\,V$                     & Dimension of feature channel, height and width. \\
$z_k^{t,e},\,\tilde{z}_k^{t,e}$ & Original and sparsified intermediate feature.       \\
$\eta^t$                        & Learning rate at iteration $t$.                         \\
$N$                             & Number of retained feature channels after sparsification. 
                \\
$R$                             & Sparsification ratio $R=N/C$.                       \\
$I_k^t$                         & Class-balanced importance vector.           \\
$S$                             & Smoothness constant.                                \\
$\beta$ & Momentum coefficient for EMA in ICS.
                           \\
$\psi^2$ & Bounded client heterogeneity.                                                \\      
$m_k^{t,e}$                     & Update error caused by sparsification.              \\ 
$\mathcal{C}_{d},\,\mathcal{C}_{s}$ & Computational costs of client-side and server-side models. \\
$P_{d},\,P_{s}$ & Numbers of client-side and server-side model parameters.  \\\hline
\end{tabular}

\end{table}

\section{System Model}\label{sec: system model}
Consider a split learning framework with a central parameter server and $K$ local clients, each equipped with a single antenna. The goal is to collaboratively train a global model $\theta = [\theta_d \,\|\, \theta_s]$, where $\theta_d$ and $\theta_s$ are the client-side and server-side model parameters, respectively, and $[\cdot \,\|\, \cdot]$ denotes concatenation. Each client $k\in\mathcal{K}=\{0,\dots,K-1\}$ holds a local dataset
$\mathcal{D}_k=\{(x_i,y_i)\}_{i=1}^{D_k}$, where
$x_i$ is an input image, $y_i\in\{0,1\}^L$ is its corresponding one-hot encoded label with the total number of classes $L$, and $D_k=|\mathcal{D}_k|$. Let the global dataset be $\mathcal{D}=\bigcup_{k\in\mathcal{K}}\mathcal{D}_k$. For each client, the local objective function is
\begin{align}
    F_{k}(\theta_{k})=\frac{1}{|\mathcal{D}_k|}\sum_{(x_i,y_i)\in\mathcal{D}_k}\mathcal{L}(f_s(f_d(x_i;\theta_{d,k});\theta_{s,k}), y_{i}),
\end{align}
where $f_{d}(\;\cdot\;;\theta_{d,k})$ and $f_{s}(\;\cdot\;;\theta_{s,k})$ representing the client-side feature extraction and server-side classification mapping parameterized by the parameters $\theta_{d,k}$ and $\theta_{s,k}$, respectively. The global objective function is formulated as $F(\theta)=\sum_{k\in\mathcal{K}}\frac{|\mathcal{D}_k|}{|\mathcal{D}|}F_k(\theta_k)$.

Fig. \ref{fig: System model} represents our system model with the numbered \emph{stages} described below. The SL protocol is executed over $T$ communication rounds. The client and the server update $\theta_{k}^{t,0}$ using the client's dataset $\mathcal{D}_k$ by employing a mini-batch stochastic gradient descent (SGD) algorithm. The update rule is therefore 
\begin{align}
    \theta_{k}^{t,e+1} = \theta_{k}^{t,e} - \eta^{t} \nabla F_{k}^{t,e}(\theta_{k}^{t,e}),
\end{align}
where $e \in \{0, \dots, E - 1\}$ is the index of local step given the total local steps $E$, and $\eta^{t}$ denotes the learning rate at global iteration $t$. Here, $\nabla F_{k}^{t,e}(\theta_{k}^{t,e})$ denotes the stochastic gradient with respect to the complete parameters $\theta^{t,e}_k$. We use an explicit subscript when referring to the gradient with respect to a specific component. In each iteration $t$ at local step $e$, for every client $k \in \mathcal{K}$, the following \emph{stages} are carried out:

\begin{figure}[!t]
    \centering
    \includegraphics[width=0.4\textwidth]{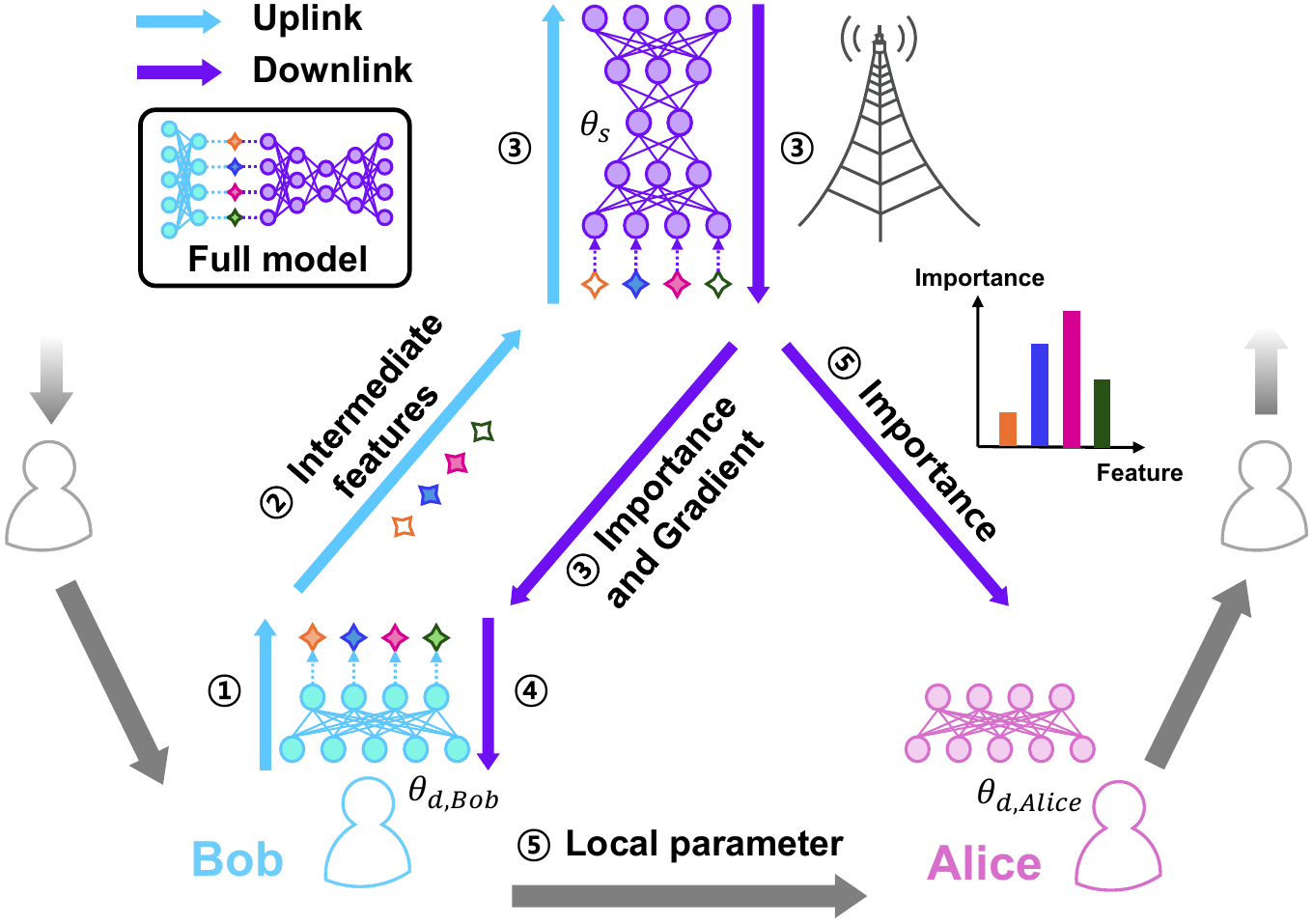}
    \caption{System model of proposed framework with the numbered \emph{stages}.}
    \label{fig: System model}
\end{figure}

\begin{enumerate}
    \item \textbf{Client-side Feature Extraction:} 
    Client $k$ draws a random mini-batch $\mathcal{B}_{k}^{t,e} \subset \mathcal{D}_{k}$ of size $B$. Let $x_{k}^{t,e}=\{x_{k,b}^{t,e}\}_{b=1}^B$ be the input data in the mini-batch $\mathcal{B}_{k}^{t,e}$ of client $k$ at iteration $t$ and local step $e$. Then, the mini-batch is processed by the local feature extractor $f_{d}$ to produce intermediate features $z_{k}^{t,e} = \bigl[z_{k,1}^{t,e}, \dots, z_{k,B}^{t,e}\bigr]^\top \in \mathbb{R}^{B\times C\times U \times V},$ where $C$ denotes the number of feature channels and $U\times V$ are the layer dimensions. This step can be expressed as $z_{k}^{t,e} = f_{d}(x_{k}^{t,e};\theta_{d,k}^{t,e})$.

    \item \textbf{Uplink Transmission:} 
    The intermediate feature $z_k^{t,e}$, together with its corresponding labels, is transmitted to the server over the wireless uplink. Prior to transmission, client $k$ applies sparsification to $z_{k}^{t,e}$ by the proposed \textbf{ICS} in accordance with the uplink capacity constraints, as described in Section \ref{sec: sparsification} later, such that the received signal can be treated as effectively noise-free.

    \item \textbf{Server-side Computation:} 
    Upon reception, the server completes the forward pass by computing 
    \begin{align}
        \hat{y}_{k}^{t,e} = f_{s}(z_{k}^{t,e};\theta_{s,k}^{t,e}) \in \mathbb{R}^{B \times L}.
    \end{align}
    The server evaluates the per-sample loss $\mathcal{L}_{vec}(\hat{y}_{k}^{t,e}, y_{k}^{t,e}) \in \mathbb{R}^B$ using the true labels $y_{k}^{t,e} \in \mathbb{R}^{B \times L}$. Then, the server computes the mini-batch loss $F_k(\theta^{t,e}_k)$ by averaging loss vector, i.e., $F_k(\theta^{t,e}_k)=\frac{1}{B}\sum_{b=1}^B\mathcal{L}_{vec,b}\in\mathbb{R}$. The server updates the model as
    \begin{align}
        \theta_{s,k}^{t,e+1}=\theta_{s,k}^{t,e}-\eta^t\nabla_{\theta_{s,k}^{t,e} }F_k(\theta_{k}^{t,e})
    \end{align}
    by performing a mini-batch SGD. The server transmits the gradient with respect to the received activation $\nabla_{z_{k}^{t,e}} F_k(\theta_k^{t,e}) = \frac{\partial F_k(\theta_k^{t,e})}{\partial z_{k}^{t,e}} \in \mathbb{R}^{B\times C\times U\times V}$ and importance vector $I_k^{t,e+1}$, described in Section \ref{sec: sparsification} later, to client $k$.

    \item \textbf{Client Update:} 
    Client $k$ uses the returned gradient $\nabla_{z_{k}^{t,e}} F_k(\theta_k^{t,e})$ to perform a local parameter update, producing new local parameters $\theta_{d,k}^{t,e+1}$ via
    \begin{align}
        \theta_{d,k}^{t,e+1}=\theta_{d,k}^{t,e}-\eta^t \nabla_{\theta_{d,k}^{t,e} }F_k(\theta_{k}^{t,e}),
    \end{align}
    where $ \nabla_{\theta_{d,k}^{t}}F_k(\theta_k^{t,e})=\left(\frac{\partial z_k^{t,e}}{\partial\theta_{d,k}^{t,e}}\right)^{\textsf{T}}\nabla_{z_k^{t,e}}F_k(\theta_k^{t,e})$.    
    After the updates, if $e<E-1$, the client $k$ and the server repeat the \emph{stages 1-4}. For $e=E-1$, the client $k$ and the server perform the \emph{stage 5}.

    \item \textbf{Transmission of Local Parameters and Importance}: 
    After completing the local updates, client $k$ transmits the updated local parameters $\theta_{d,k}^{t,E}$ to client $k+1$. Separately, the server transmits importance vector to client $k+1$ for use in the next iteration. When the model is updated by the last client $K-1$, the updated local 
    parameters are transmitted to client $k=0$, becoming $\theta_{d,0}^{t+1,0}$, while the server also forwards the corresponding importance vector to client $0$.
\end{enumerate}

After these $T$ iterations, the training process concludes. 

\section{Importance-Aware Class-
balanced\\Feature Sparsification}\label{sec: sparsification}

\subsection{Conventional Grad-CAM for Model Interpretability}
Grad-CAM is a gradient-based attribution method that identifies which regions of an input image most strongly influence a model’s prediction. It works by backpropagating the gradients of a target class score into a chosen intermediate convolutional layer and computing importance weights for each feature channel. These weights are then combined with the corresponding activation maps to generate a coarse localization map that highlights the image regions most relevant to the decision. Since Grad-CAM is applied to a fixed, pre-trained model, it assumes that the model parameters are already converged and thus measures the contribution of intermediate features to the final output without considering ongoing learning iterations.

The following procedure describes how Grad-CAM evaluates the relevance of a particular intermediate feature. Because the method is applied post-hoc to a pre-trained network, the training iteration index is omitted.

\begin{enumerate}
    \item \textbf{Feature Extraction:}  
    During the forward propagation, let $A_c \in \mathbb{R}^{U \times V}$ denote the layer feature map corresponding to the $c$-th feature channel of a selected intermediate layer, with $c \in \{1,\dots,C\}$. Here, $U$ and $V$ are the layer dimensions. Each feature channel encodes different semantic components of the input.

    \item \textbf{Gradient Calculation:}  
    For a target class $\ell$, compute the partial derivatives of the predicted score $\hat{y}_\ell$ with respect to each feature map:
    \begin{align}\label{eq:grad_cam}
        \frac{\partial \hat{y}_\ell}{\partial A_c} \in \mathbb{R}^{U \times V}.
    \end{align}
    These gradients quantify how changes in each feature map affect the predicted score for class $\ell$. Since each feature channel in CNN models captures different characteristics of the input in the training process, the gradients reflect the sensitivity of the class score to these distinct features, thereby indicating their relative importance.

    \item \textbf{Feature Channel-wise Importance Weighting:}  
    Aggregate the layer gradient information to produce a single importance coefficient for each feature channel:
    \begin{align}
        \alpha_{c,\ell} = \frac{1}{U V} \sum_{u=1}^{U} \sum_{v=1}^{V} \frac{\partial \hat{y}_\ell}{\partial A_c(u,v)}.
    \end{align}
    The scalar $\alpha_{c,\ell}$ represents the overall contribution of feature channel $c$ to the class score $\hat{y}_\ell$.

    \item \textbf{Saliency Map Generation:} The class-discriminative localization map, $P_{\ell} \in \mathbb{R}^{U \times V}$, is first computed as a weighted linear combination of the feature maps $A_c$ using their corresponding weights $\alpha_{c,\ell}$:
    \begin{align}
        P_{\ell} = \sum_{c=1}^{C} \alpha_{c,\ell} A_c.
    \end{align}
    To produce the final saliency map and focus only on features that have a positive influence on the prediction, a $\text{ReLU}$ activation is applied:
    \begin{align}
        H_{\ell} = \text{ReLU}\left(P_{\ell}\right).
    \end{align}
    The resulting $H_{\ell}$ is the saliency map used for visualization, highlighting the specific image regions that positively contributed to the prediction for class $\ell$.
\end{enumerate}
In section \ref{sec: numerical results}, these saliency maps are employed to validate the efficacy of our proposed framework.

\subsection{Importance-Aware Class-Balanced Sparsification}\label{subsec: sparsification}
\begin{algorithm}[t]
\caption{Proposed \textbf{ICS} SL framework}
\label{alg:sl-ics}
\begin{algorithmic}[1]
\Require Clients $\mathcal{K}=\{0,\dots,K-1\}$, rounds $T$, batch size $B$, ratio $R$, momentum $\beta$, learning rate $\eta^t$
\State Initialize $\theta_{d,0}^{0}$, $\theta_{s,0}^{0}$, $M_{k,\ell}^{0}\leftarrow0\in\mathbb{R}^{C}$, $\forall k,\ell$;
\For{$t=0$ \textbf{to} $T-1$}
  \For{$k=0$ \textbf{to} $K-1$} 
    \Statex \textbf{AT THE CLIENT-SIDE}
    \State Draw $\mathcal{B}_k^t$ of size $B$ \label{alg: 4}
    \State Obtain intermediate feature as $z_k^t \leftarrow f_d(x_k^t;\theta_{d,k}^t)$ \label{alg: 5}
    \State Sparsify $z_k^t$ using \eqref{eq:sparsification} \label{alg: 6}
    \State Transmit $(\tilde{z}_k^t, y_k^t)$ to server \label{alg: 7}
    \Statex \textbf{AT THE SERVER-SIDE}
    \State Obtain logit as $\hat{y}_k^t \leftarrow f_s(\tilde{z}_k^t;\theta_{s,k}^t)$ \label{alg: 8}
    \State Obtain $\nabla_{\tilde z_k^t}F_k(\theta_k^t)$ using \eqref{eq: chain rule} \label{alg: 9}
    \State Update server-side parameters as 
    \Statex \qquad\qquad$\theta_{s,k+1}^t \leftarrow \theta_{s,k}^t - \eta^t \nabla_{\theta_{s,k}^t}F_k(\theta_k^t)$ \label{alg: 10}
    \For{$b=1$ to $B$} \label{alg: 11}
        \State Update $\zeta_{k+1,b}^{t}$ using \eqref{eq: label-specific importance} \label{alg: 12}
    \EndFor \label{alg: 13}
    \For{$\ell=1$ to $L$} \label{alg: 14}
        \State Update $\zeta_{k+1,\ell}^{t}$ using \eqref{eq:class-avg} \label{alg: 15}
        \State Update $M_{k+1,\ell}^{t}$ using \eqref{eq:ema} \label{alg: 16}
        
    \EndFor \label{alg: 17}
    \State Obtain $I_{k+1}^{t}$ through \eqref{eq:importance_vector} \label{alg: 18}
    \State Transmit $\nabla_{\tilde z_k^t}F_k(\theta_k^t)$ to client $k$ \label{alg: 19} 
    \If{$k+1 \leq K-1$} \label{alg: 20}
        \State Transmit $I_{k+1}^{t}$ to client $k+1$ \label{alg: 21}
    \Else \label{alg: 22}
        \State Transmit $I_{k+1}^{t}$ to client $0$ \label{alg: 23}
    \EndIf \label{alg: 24}
    
    \Statex \textbf{AT THE CLIENT-SIDE}
    \State Update client-side parameters as \label{alg: 25}
    \Statex \qquad\qquad $\theta_{d,k+1}^t \leftarrow \theta_{d,k}^t - \eta^t \nabla_{\theta_{d,k}^t}F_k(\theta_{k}^t)$ 
    \If{$k+1 \leq K-1$} \label{alg: 26}
        \State Send $\theta_{d,k+1}^t$ to client $k+1$ \label{alg: 27}
    \Else \label{alg: 28}
        \State Send $\theta_{d,k+1}^t$ to client $0$ \label{alg: 29}
    \EndIf \label{alg: 30}
  \EndFor \label{alg: 31}
\EndFor \label{alg: 32}
\end{algorithmic}
\end{algorithm}
The intermediate feature outputs transmitted from clients to the server are typically high-dimensional, driving substantial communication cost. A naive remedy to reduce communication burden is to select top feature elements by magnitude, since large features often dominate the output. However, although high-magnitude features can significantly influence the overall output, they do not necessarily capture the performance-critical aspects. This is because the overall output reflects contributions from both the true label and non-target labels, meaning that high-magnitude features might include features unrelated to the true label, whereas low-magnitude features might include crucial features related to the true label.

To address this, we take inspiration from Grad-CAM, which explicitly measures the contribution of intermediate feature channels to the score of the true class. By leveraging such class-specific importance information, we can more reliably identify and preserve the feature components that are most critical for correct prediction. In standard usage, Grad-CAM is computed post-hoc by performing a full forward and backward pass through both client-side and server-side models to obtain gradients with respect to the intermediate features. However, in SL, the client only executes its forward pass before transmission and thus cannot compute the real-time Grad-CAM weights prior to sending intermediate features.

To circumvent this, we propose using the importance scores from the previous communication round as a proxy for the current round. In the following description, the numbers in bold parentheses \textbf{(Line $\cdot$)} indicate the corresponding line numbers in Algorithm~\ref{alg:sl-ics}. As described in Section~\ref{sec: system model}, client $k$ first draws a mini-batch and obtains the intermediate feature $z_k^t$\footnote{In Subsection \ref{subsec: sparsification}, for notational simplicity, we omit local iteration and describe a single forward-backward process per iteration $t$ at client $k$.} through client-side feature extraction \textbf{(Lines \ref{alg: 4}--\ref{alg: 5})}. Before transmission, client $k$ uses the importance vector $I_k^{t} \in \mathbb{R}^{C}$ received from the previous communication round. The client then applies a sparsification operator $\mathcal{S}(\cdot)$ that, for each sample $b$, ranks the entries of $z_{k,b}^t$ by leveraging the feature channel-level scores in $I_k^{t}$ and selects the top-$N$ elements \textbf{(Line \ref{alg: 6})}. The sparsification ratio is defined as $R = \frac{N}{C}$. The resulting transmitted intermediate feature is
\begin{align}\label{eq:sparsification}
    \tilde{z}_k^t = \mathcal{S}(z_k^t; I_k^{t}) \in \mathbb{R}^{B \times C \times U \times V},
\end{align}
where $\mathcal{S}$ preserves only $N$ feature channel entries according to the importance-weighted ranking and zeroes out the rest.

The client then transmits the sparsified intermediate feature $\tilde{z}_k^t$ to the server \textbf{(Line \ref{alg: 7})}. Upon receiving it, the server completes the forward pass and obtains $\hat{y}_k^t$ by processing $\tilde{z}_k^t$ through its subsequent layers, parameterized by $\theta_{s,k}^t$, to generate the output logits \textbf{(Line \ref{alg: 8})}. Let $\hat{y}_{k,b,\ell}^{t}$ denote the logit for class $\ell$ corresponding to the $b$-th sample in the mini-batch of client $k$ at iteration $t$.

To train the model, the server computes a loss function using cross-entropy by comparing the predicted logits with the ground-truth labels:
\begin{align}
    \mathcal{L}(\hat{y}_{k,b}^t, y_{k,b}^t) = -\sum_{\ell=1}^L y_{k,b,\ell}^{t}  \log \left( \frac{\exp(\hat{y}_{k,b,\ell}^{t})}{\sum_{j=1}^{L} \exp(\hat{y}_{k,b,j}^{t})} \right). \label{eq: loss_function}
\end{align}

During backpropagation, the gradient of the loss with respect to the intermediate feature is computed using the chain rule. For the gradient of mini-batch sample $x_{k,b}^t$ in $\tilde{z}_{k,b}^t$, denoted as $\frac{\partial F_k(\theta_k^t)}{\partial \tilde{z}_{k,b}^t}$, we have \textbf{(Line \ref{alg: 9})}
\begin{align}\label{eq: chain rule}
    \frac{\partial F_k(\theta_k^t)}{\partial \tilde{z}_{k,b}^t} = \frac{1}{B} \sum_{\ell=1}^L \frac{\partial \mathcal{L}}{\partial \hat{y}_{k,b,\ell}^{t}} \frac{\partial \hat{y}_{k,b,\ell}^{t}}{\partial \tilde{z}_{k,b}^t}.
\end{align}
Then, the server updates the model by performing a mini-batch SGD \textbf{(Line \ref{alg: 10})}.
In this formulation, $\frac{\partial \hat{y}_{k,b,\ell}^{t}}{\partial \tilde{z}_{k,b}^t}\in\mathbb{R}^{C\times U\times V}$ is equivalent to calculating the Grad-CAM over $C$ feature channels, which quantifies the sensitivity of the logit for class $\ell$ at the layer information $(u,v)$ of sample $b$. 

To derive a label-specific importance measure for each feature channel, the computed gradients are aggregated over the layer dimensions \textbf{(Lines \ref{alg: 11}--\ref{alg: 13})}:
\begin{align}\label{eq: label-specific importance}
    \zeta_{k+1,b}^{t} = \frac{1}{W} \sum_{u=1}^{U}\sum_{v=1}^{V} \frac{\partial \hat{y}_{k,b,\ell_b}^{t}}{\partial \tilde{z}_{k,b}^t(u,v)}\in\mathbb{R}^{C},
\end{align}
where $\hat{y}_{k,b,\ell_b}^{t}$ denotes the logit corresponding to the true class $\ell_b$ for the $b$-th sample of client $k$ at time $t$. $W = U \times V$ is the normalization factor, and $U$ and $V$ denote the layer dimensions of $\tilde{z}_k^t$. This aggregation yields a feature channel-wise importance score that reflects the contribution of each feature channel to the prediction of the true class.

As the unified importance vector aims to capture intermediate feature channels that are universally significant across all labels, it is ideally computed under an i.i.d. assumption regarding class distributions. However, the data held by clients in SL environments typically exhibit non-i.i.d. characteristics, leading to potentially biased estimations if naively averaged across samples.

To address this challenge, we first compute the class-specific importance vectors by averaging per-sample importance vectors obtained in \eqref{eq: label-specific importance} over the corresponding subsets of samples belonging to each label. Formally, for client $k$ at iteration $t$, we define the subset of samples with label $\ell$ as $\mathcal{B}_{k,\ell}^{t} = \{x_{k,b}^t \mid y_{k,b,\ell}^{t} = 1\}$ and compute the class-specific importance vector $\zeta_{k+1,\ell}^{t} \in \mathbb{R}^{C}$ as \textbf{(Line \ref{alg: 15})}
\begin{align}\label{eq:class-avg}
    \zeta_{k+1,\ell}^{t}=
    \frac{1}{B_{k,\ell}^{t}}
    \sum_{b\in\mathcal{B}_{k,\ell}^{t}}
    \zeta_{k+1,b}^{t},
\end{align}
where $B_{k,\ell}^{t} = |\mathcal{B}_{k,\ell}^{t}|$.
To ensure temporal stability and mitigate abrupt changes between consecutive iterations, we maintain an exponential moving average (EMA) memory, denoted as $M_{k,\ell}^{t}\in\mathbb{R}^{C}$, for each client $k$ and label $\ell$. The EMA memory is updated using a momentum coefficient $\beta\in(0,1)$ as follows \textbf{(Line \ref{alg: 16})}: 
\begin{align}\label{eq:ema}
    M_{k+1,\ell}^{t}=\beta M_{k,\ell}^{t}+
    (1-\beta)\zeta_{k+1,\ell}^{t},
\end{align}  
where $M$ is initialized with $0$ at the beginning.

Finally, the importance vector is the uniform average of the current EMA memories across all classes \textbf{(Line \ref{alg: 18})}:
\begin{align}\label{eq:importance_vector}
    I_{k+1}^{t}=\frac{1}{|\mathcal{P}_{k+1}^{t}|}
    \sum_{\ell\in\mathcal{P}_{k+1}^{t}}
    M_{k+1,\ell}^{t},
\end{align}
where $\mathcal{P}_{k+1}^{t}=\bigl\{\ell\mid M_{k+1,\ell}^{t}\neq0\}.$
This averaging scheme balances potential class imbalance inherent in non-i.i.d. settings by assigning equal weight to every class that has contributed at least once to the feature channel importance scores. After computing the aggregated importance vector $I_{k+1}^{t}$, the server returns the activation gradient $\nabla_{\tilde z_k^t}F_k(\theta_k^t)$ to client $k$ for client-side backpropagation and forwards $I_{k+1}^{t}$ to client $k+1$ for the next training step \textbf{(Lines \ref{alg: 19}--\ref{alg: 21})}. Note that when the model is updated by the last client $K-1$, the server forwards $I_{0}^{t+1}$ to client $0$ for the next iteration \textbf{(Line \ref{alg: 23})}. Meanwhile, client $k$ uses the returned activation gradient $\nabla_{\tilde z_k^t}F_k(\theta_k^t)$ to update the client-side model parameters, resulting in $\theta_{d,k+1}^t$ \textbf{(Line \ref{alg: 25})}. After the client-side update, the updated client-side model parameters are forwarded to the next client in the sequential order, or to client $0$ if $k=K-1$ \textbf{(Lines \ref{alg: 26}--\ref{alg: 30})}.
Fig. \ref{fig: System model} provides an overall framework of our proposed scheme.
\begin{remark}\label{remark: architecture_agnostic}
Although \textbf{ICS} is described above using CNN intermediate features of shape $\mathbb{R}^{B\times C\times U\times V}$, its importance estimation does not particularly rely  on  convolution-specific operation. It  requires only the gradient of the true-class logit with respect to the intermediate feature, aggregated over the non-feature axes to obtain a feature-wise importance score. Hence, \textbf{ICS} applies to any architecture in which a sparsifiable feature dimension can be identified and the remaining axes can be pooled. For a transformer-based model with a token embedding tensor $z\in\mathbb{R}^{B\times d_{TOK}\times d_{EMB}}$, the embedding dimension plays the role of the CNN feature channel and the importance is aggregated over the token axis, as empirically validated in Section~\ref{sec: numerical results}.
\end{remark}

\section{Theoretical Analysis}
\subsection{Communication Overhead Analysis}
\label{subsec:overhead_analysis}
We first analyze the uplink overhead, since reducing client-side communication burden is particularly important in wireless SL. In conventional SL, each client transmits the full intermediate feature $z_k^t\in\mathbb{R}^{B\times C\times U\times V}$, yielding a per-client uplink overhead of $\mathcal{O}(BCUV)$. With the sparsification ratio $R=\frac{N}{C}$, transmitting only $N$ out of $C$ feature channels reduces this to $\mathcal{O}(BRCUV)$.

For communication-efficient SL methods whose selection pattern is determined at the client side \cite{zhang2024federated, zhou2024mask, zheng2023reducing}, the server cannot infer the retained channels from the received features alone. Hence, additional mask index information must be transmitted. Since the sparsification pattern differs across samples, the index overhead scales with the mini-batch size, yielding $\mathcal{O}(BRCUV+BRC\log C)$. For \textbf{SplitFC} \cite{oh2025communication}, the selection is made over the $CUV$ feature positions rather than the feature channels, so its index is generated over $CUV$ positions, yielding $\mathcal{O}(BRCUV+RCUV\log(CUV))$. \textbf{FedLite} \cite{wang2022fedlite} instead transmits a codebook together with the codeword indices of the clustered sub-vectors. With $Q$ centroids and $n_{sub}$ sub-vectors per sample, the codebook requires $\mathcal{O}(Qd_{sub})$ and the per-sample indices require $\mathcal{O}(Bn_{sub}\log Q)$, yielding a per-client uplink overhead of $\mathcal{O}(Qd_{sub}+Bn_{sub}\log Q)$, where $d_{sub}$ is the sub-vector dimension.

In contrast, \textbf{ICS} uses the server-generated importance vector $I_k^t\in\mathbb{R}^{C}$, which is shared with the client before sparsification. Since the client selects the retained elements using a deterministic top-$N$ rule based on $I_k^t$, the server can infer the retained elements directly from $I_k^t$. Thus, under the same sparsification ratio, \textbf{ICS} achieves a per-client uplink overhead of $\mathcal{O}(BRCUV)$ without transmitting any additional mask index information.

\begin{table}[t]
\caption{Runtime comparison of communication-efficient SL methods.}
\label{tab:runtime_overhead}
\centering
\renewcommand{\arraystretch}{1.2}
\begin{tabular}{lcc}
\hline
\textbf{Method} 
& \makecell{\textbf{Client-side}\\ \textbf{sparsification (msec/batch)}}
& \makecell{\textbf{Server-side additional}\\ \textbf{runtime (msec/batch)}} \\ \hline
RS  & 0.4133 & 0.0 \\
TS  & 0.5099 & 0.0 \\
RTS  & 0.5357 & 0.0 \\ 
FedLite  & 31.7882 & 0.0 \\
SplitFC  & 0.8106 & 0.0 \\
\textbf{ICS}  & 0.0928 & 6.1702 (effectively 0.0) \\\hline
\end{tabular}
\end{table}
\subsection{Computational Complexity Analysis}
All methods share the same vanilla SL forward and backward pass on the client-side and server-side, whose costs are denoted by $\mathcal{C}_d$ and $\mathcal{C}_s$, respectively. Thus, we omit them and focus on the additional preprocessing for feature selection, which consists of 1) computing the selection scores; 2) selecting the retained features. The gathering of selected features into the retained tensor is also omitted as common to all methods.

\textbf{RS} \cite{zhang2024federated} randomly generates the mask without accessing the feature values, requiring only $\mathcal{O}(C)$ to select $N$ out of $C$ feature channels. \textbf{TS} \cite{zhou2024mask} and \textbf{RTS} \cite{zheng2023reducing} scan the intermediate feature at a cost of $\mathcal{O}(BCUV)$, then rank the per-channel scores for each sample, yielding $\mathcal{O}(BCUV+BC\log C)$. \textbf{SplitFC} \cite{oh2025communication} computes feature-vector-wise statistics at a cost of $\mathcal{O}(BCUV)$ and then selects over the $CUV$ feature elements, yielding $\mathcal{O}(BCUV+CUV\log(CUV))$. Since its selection is made over feature elements rather than feature channels, its selection cost scales with $CUV$ rather than $C$, making it heavier than the feature channel-wise methods. \textbf{FedLite} \cite{wang2022fedlite} applies $K$-means clustering to the intermediate feature, comparing every sub-vector against $Q$ centroids over $J$ iterations, yielding $\mathcal{O}(QJBCUV)$. Its iterative clustering thus makes it the most expensive, exceeding the selection-based methods by a factor $(QJ)$.

In contrast, \textbf{ICS} applies a deterministic top-$N$ rule to the length-$C$ importance vector $I_k^t$, so its additional client-side cost is only $\mathcal{O}(C)$, far below the $\mathcal{O}(BCUV)$ scan of the other methods. Although this matches the order of \textbf{RS}, \textbf{ICS} is even lighter in practice, since the client only reads the server-generated vector $I_k^t$ without generating a mask from the features. The importance computation is instead shifted to the server, which reuses the forward-pass logits to compute the importance measure at a cost $\mathcal{C}_{IM}$. While this estimation is performed sequentially after the server-side update in our implementation, this estimation does not depend on the updated parameters and can thus be run in parallel with the update. Since it backpropagates only the true-class logit to the intermediate feature, it traverses a strictly shorter path than the full-loss backward of the update.  The subsequent spatial pooling, class-wise averaging, and EMA updates operate only on the $C$ feature channels and incur minimal computational overhead, making their cost negligible compared with that of the neural network. The importance estimation is therefore completed within the update and fully hidden behind it, reducing its effective additional runtime to zero. This allocation aligns with the SL principle, keeping the resource-constrained client nearly free of selection overhead while delegating the heavier importance computation to the resourceful server.

To complement this analysis, we measure the actual runtime overhead of each method under the setting in Section~\ref{sec: numerical results}. Table~\ref{tab:runtime_overhead} reports the per-batch client- and server-side overhead, excluding the common forward and backward computations so that only the method-specific cost is compared. 

\subsection{Convergence Analysis}\label{sec: convergence analysis}
For theoretical derivation, we make the following assumptions typically made in analyzing the distributed learning family\cite{li2023convergence,kim2025communication, huh2025federated}.
\begin{assumption}\label{as: sparsification}
    Let $\mathcal{S} : \mathbb{R}^{C} \rightarrow \mathbb{R}^{C}$ be a proposed sparsification operator and $N$ be the number of non-zero elements satisfying $0 < N < C.$ Then, for all $x \in \mathbb{R}^{C}$, $\mathcal{S}$ is an unbiased operator and its variance satisfies $\mathbb{E}\left[\|\mathcal{S}(x)-x\|_2^2\right]\leq\left(1-\frac{N}{C}\right)\|x\|_2^2.$
\end{assumption}

\begin{assumption}\label{as: smooth}
    For all $\theta$, the local objective function $F_k(\theta)$ is $S$-smooth, i.e., $\|\nabla F_k(\theta)- \nabla F_k(\theta')\|_2\leq S\|\theta-\theta'\|_2,$
    thereby $F_k(\theta')\leq F_k(\theta)+\nabla F_k(\theta)^\emph{\textsf{T}}(\theta'-\theta)+\frac{S}{2}\|\theta'-\theta\|_2^2,$
    where $\nabla F_k(\theta_k) = [\nabla_{\theta_{d}}F_k(\theta_k) \| \nabla_{\theta_{s}}F_k(\theta_k)]$. The global objective function $F(\theta)$, being the average of the local objectives, is then also $S$-smooth, and lower bounded as $F(\theta) \geq F(\theta^*)$ $\forall\theta$.
\end{assumption}

\begin{assumption}\label{as:local gradient unbiasedness}
The stochastic gradient at each client is unbiased such as $\mathbb{E}[\nabla F_{k}(\mathcal{B}^e_k;\theta)] = \nabla F_k(\theta)$ and its variance is bounded as $\mathbb{E}[\|\nabla F_{k}(\mathcal{B}^e_k;\theta)-\nabla F_k(\theta)\|_2^2]\leq \frac{\sigma^2}{B}$ for all $\theta$, client $k$ and local steps $e$ with a constant $\sigma^2>0$.
\end{assumption}

\begin{assumption}\label{as: client heterogeneity}
    There exists a constant $\psi^2$ such that $\frac{1}{K}\sum_{k=0}^{K-1}\|\nabla F_k(\theta)-\nabla F(\theta)\|_2^2\leq\psi^2$. 
\end{assumption}

\begin{assumption}\label{as: norm boundedness}
The expected squared norm of the gradient and intermediate feature at each client are upper bounded as $\mathbb{E}[\|\nabla F_{k}(\theta_k^e)\|_2^2] \leq G^2$ and $\mathbb{E}[\|z_{k,b}^e\|_2^2]\leq \delta^2$ for all client $k$, local steps $e$, sample $b$, respectively.
\end{assumption}

\begin{theorem}\label{theorem: convergence}
    Under \textbf{Assumptions \ref{as: sparsification}, \ref{as: smooth}, \ref{as:local gradient unbiasedness}, \ref{as: client heterogeneity}, \ref{as: norm boundedness}} and given the sparsification ratio $R=\frac{N}{C}$, proposed SL with learning rate $\eta^t=\frac{1}{KE\sqrt{T}}$ has a convergence rate as below:
    \begin{align}
        &\mathbb{E}\left[\frac{1}{T}\sum_{t=0}^{T-1}\|\nabla F(\theta^t)\|_2^2\right]\leq \frac{2}{\sqrt{T}}\mathbb{E}[F(\theta^{0})-F(\theta^{*})]+ 2\psi^2\nonumber\\
        &+ \frac{2S^2}{3T}\left(\frac{\sigma^2}{B}+G^2\right)
        +\frac{2S}{\sqrt{T}}\Bigg(\frac{\sigma^2}{B}+G^2+\Lambda^2\left(1-R\right)B\delta^2\Bigg),\nonumber
\end{align}
where $\Lambda^2=\Lambda_0^2(\Lambda_1^2+\Lambda_2^2\Lambda_3^2)$. Here, $\Lambda_0$, $\Lambda_1$, $\Lambda_2$ and $\Lambda_3$ are spectral norm upper-bound of $\frac{\partial\mathcal{L}}{\partial \hat{y}_{k,b}^{t,e}}$, $\frac{\partial^2\hat{y}^{t,e}_{k,b,\ell}}{\partial z_k^{t,e}\partial\theta_{s,k}^{t,e}}$, $\frac{\partial^2 \hat{y}^{t,e}_{k,b,\ell}}{\partial \left(z_k^{t,e}\right)^2}$ and $\frac{\partial z_k^{t,e}}{\partial \theta_{d,k}^{t,e}}$ $\forall k,t,\ell$, respectively.
\end{theorem}
\begin{IEEEproof}
    See Appendix \ref{apdx of theorem: convergence}.
\end{IEEEproof}
The bound yields two practical insights. First, as in FL, client data heterogeneity directly limits convergence. In particular, the term $2\psi^{2}$ does not vanish with $T$. Second, unlike FL, the uplink overhead in SL scales with the mini-batch size $B$ because each round transmits $B$ intermediate features. Under a fixed communication budget, increasing $B$ reduces the stochastic gradient variance $\frac{\sigma^{2}}{B}$ but forces a lower sparsification ratio $R$, which increases the sparsification error term $\Lambda^{2}(1-R)\,B\,\delta^{2}$. Hence, a larger $B$ is not always better in SL, and choosing $B$ carefully is crucial. Empirical evidence for this trade-off appears in Section~\ref{sec: numerical results}.
\begin{remark}
 Our convergence analysis is established under mini-batch SGD, where the effect of sparsification enters through a single gradient-variance term. Adaptive optimizers such as Adam introduce additional moment-based dynamics, which require a different technical treatment and therefore fall outside the scope of our SGD-focused analysis. However, importantly, this does not affect the main insight: the variance–sparsification tradeoff identified in our analysis is optimizer-independent and also arises when Adam is used, because any sparsification-induced deviation in the stochastic gradient propagates through the update rule regardless of how the optimizer normalizes it. Moreover, the proposed sparsification mechanism is fully compatible with adaptive optimizers, as it operates directly on intermediate feature tensors and does not depend on optimizer-specific structures.
\end{remark}

\section{Extension to Parallel Split Learning}\label{sec:psl_extension}
The system model in Section~\ref{sec: system model} considers the standard sequential SL protocol, where clients communicate with the server one at a time. In this setting, the importance vector computed at the server after processing client $k$ is transmitted to the next client $k+1$ and used for sparsifying its intermediate features. However, this sequential importance-forwarding rule is not directly applicable to PSL, where multiple clients upload their intermediate features to the server within the same communication round. In this section, we extend the proposed \textbf{ICS} framework to PSL. The main difference lies in how the importance vector is constructed and shared. Instead of maintaining a client-wise importance vector that follows the sequential client order, the server maintains a shared iteration-wise importance vector $I^{t}\in\mathbb{R}^{C}$, computed by aggregating Grad-CAM-based importance scores over the mini-batches of all participating clients. This shared vector is then broadcast to all clients and used for sparsification in the next communication round.

Each client $k\in\mathcal{K}$ draws a
mini-batch $\mathcal{B}_{k}^{t}$ of size $B$ at iteration $t$ and computes the intermediate feature as $z_{k}^{t} = f_{d}(x_{k}^{t};\theta_{d,k}^{t})
    \in \mathbb{R}^{B\times C\times U\times V}.$
Using the shared importance vector received from the previous round $I^{t}\in\mathbb{R}^C$, client $k$ sparsifies its intermediate feature as
\begin{align}
    \tilde{z}_{k}^{t} = \mathcal{S}(z_{k}^{t}; I^{t})
    \in \mathbb{R}^{B\times C\times U\times V},
    \label{eq:psl_sparsification}
\end{align}
where $\mathcal{S}(\cdot)$ is the sparsification operator defined in Section~\ref{subsec: sparsification}. The sparsified intermediate feature $\tilde{z}_{k}^{t}$ and the corresponding labels are then transmitted to the server.

Upon receiving the sparsified intermediate features from all participating clients, the server completes the forward pass for each client as 
\begin{align}
    \hat{y}_{k}^{t} = f_{s}(\tilde{z}_{k}^{t};\theta_{s}^{t})
    \in \mathbb{R}^{B\times L}.
\end{align}
Using the same loss function in \eqref{eq: loss_function}, the server
computes the client-wise loss $\mathcal{L}(\hat{y}_{k,b}^t,y_{k,b}^t)$, from which it obtains the server-side gradient $\nabla_{\theta_{s,k}^t}F_k(\theta_{k}^{t})$ and the activation gradient $\nabla_{\tilde{z}_{k}^{t}}F_k(\theta_k^t)$. The server-side gradients are aggregated to update the server-side model as
\begin{align}
    \theta_s^{t+1} = \theta_s^t-\eta^t\frac{1}{K}\sum_{k\in\mathcal{K}}\nabla_{\theta_{s}^t}F_k(\theta_{k}^{t}).
    \label{eq:psl_server_update}
\end{align}
The server also updates the shared importance vector for the next iteration $t+1$. Unlike the sequential SL setting, where class-specific importance is updated from the mini-batch of a single client, PSL allows the server to update class-specific importance using the mini-batches of all participating clients. Following the notation in Section~\ref{subsec: sparsification}, $\mathcal{B}_{k,\ell}^{t}$ denotes the subset of samples with label $\ell$ in the mini-batch of client $k$. Then, the total number of samples with label $\ell$ across the participating clients is given by $B_{\ell}^{t} = \sum_{k\in \mathcal{K}} |\mathcal{B}_{k,\ell}^{t}|.$

For each sample $x_{k,b}^{t}\in\mathcal{B}_{k,\ell}^{t}$, the server first computes the sample-wise importance vector following
\eqref{eq: label-specific importance}. The class-specific importance vector for PSL is then obtained by averaging over all samples of class $\ell$ across the participating clients:
\begin{align}
    \zeta_{\ell}^{t+1}=\frac{1}{B_{\ell}^{t}}\sum_{k\in\mathcal{K}}\sum_{x_{k,b}^{t}\in\mathcal{B}_{k,\ell}^{t}}\zeta_{k,b}^{t+1}. 
\end{align}
This aggregation uses the collection of samples
available at the server in the same PSL iteration, thereby providing a more balanced estimate of class-specific feature importance than a client-wise estimate under non-i.i.d. data distributions.
As in Section~\ref{subsec: sparsification}, to ensure temporal stability, the server maintains an EMA memory $M_{\ell}^{t}\in\mathbb{R}^{C}$ for each class $\ell$ and updates it as
\begin{align}
    M_{\ell}^{t+1}=\beta M_{\ell}^{t}+(1-\beta)\zeta_{\ell}^{t+1},\label{eq:psl_ema}
\end{align}
where $\beta\in(0,1)$ is the momentum coefficient and $M_\ell^t$ is initialized with $0$ at the beginning $t=0$. Finally, the shared class-balanced importance vector is obtained by averaging the EMA memories over the observed classes:
\begin{align}
    I^{t+1}=\frac{1}{|\mathcal{P}^{t+1}|}\sum_{\ell\in\mathcal{P}^{t+1}}M_{\ell}^{t+1},
    \label{eq:psl_importance_vector}
\end{align}
where $\mathcal{P}^{t+1}=\{\ell\mid M_{\ell}^{t+1}\neq0\}$.
After the server-side computation, the server sends the activation gradient $\nabla_{\tilde{z}_{k}^{t}}F_k(\theta_k^t)$ individually to each client $k\in\mathcal{K}$ for client-side model updates. Each client then updates its client-side model parameters in parallel as
\begin{align}
    \theta_{d,k}^{t+1}=\theta_{d,k}^{t}-\eta^t\nabla_{\theta_{d,k}^t}F_k(\theta_k^t),\label{eq:psl_client_update}
\end{align}
where $ \nabla_{\theta_{d,k}^{t}}F_k(\theta_k^t)=\left(\frac{\partial z_k^{t}}{\partial\theta_{d,k}^{t}}\right)^{\textsf{T}}\nabla_{\tilde{z}_k^{t}}F_k(\theta_k^t)$.
In addition, the server broadcasts the shared importance vector $I^{t+1}$ to all clients, which is used for sparsifying intermediate features in the next iteration $t+1$. 

Next, for the client-side updates synchronization, the server aggregates the updated client-side parameters as
\begin{align}
    \theta_d^{t+1}
    =
    \frac{1}{K}
    \sum_{k\in\mathcal{K}}
    \theta_{d,k}^{t+1}.
    \label{eq:PSL_client_aggregation}
\end{align}
The aggregated client-side model $\theta_d^{t+1}$ is then distributed to all clients for the next iteration. Note that this client-side synchronization step does not change the proposed ICS sparsification and importance-update procedures.

\section{Numerical Results}\label{sec: numerical results}
In this section, we present simulation results to demonstrate the effectiveness of the proposed \textbf{ICS} framework. All experiments are conducted using Python 3.8 on an Ubuntu server equipped with NVIDIA GeForce RTX 3090 GPUs. The evaluation is conducted on CIFAR-100 \cite{krizhevsky2009learning}, which contains 60,000 color images of size $32\times32$ grouped into 100 fine-grained classes (with $600$ images per class). Among these, 50,000 images are used for training and 10,000 for testing. The 100 classes are further organized into 20 superclasses, each comprising five semantically related classes. 

In i.i.d. setting, the training data is evenly and randomly distributed across all clients, ensuring that each client possesses a representative subset of the entire dataset. Moreover, in non-i.i.d. setting, the training data is partitioned across 10 local clients using a Dirichlet distribution. Specifically, for each class $\ell$, we sample a probability vector $\boldsymbol{p}_\ell \sim \mathrm{Dirichlet}(\alpha)$ and assign the samples of class $\ell$ to client $k$ in proportion to $p_{\ell,k}$. This procedure induces varying class compositions across clients controlled by the concentration parameter $\alpha$, i.e., smaller $\alpha$ yields more skewed, heterogeneous splits. 

Mini-batches of size 512 are drawn independently at each client. Both client and server-side models are optimized using SGD with an initial learning rate of 0.01, momentum of 0.9, and weight decay of $5\times10^{-4}$. A cosine annealing schedule is applied to decay the learning rate. For reproducibility, random seeds are fixed for data partitioning and model initialization. The client-side model is a lightweight ResNet variant comprising two residual layers, totaling approximately 675K parameters. The server-side model employs a deeper ResNet-style architecture with three groups of residual blocks, amounting to roughly 19.94M parameters. Unless otherwise specified, the sparsification ratio is set to $R=0.2$.

Regarding the baseline schemes in the simulation, we consider the following frameworks. To ensure a comprehensive comparison, we select one representative method from each major category of communication-efficient SL feature selection, namely random sampling, magnitude-based selection and its randomized variant, feature-statistics-based selection, and clustering-based compression.For \textbf{FedLite}, the compressed payload consists of codebooks and codewords generated by clustering intermediate features. Since its discrete compression does not always allow an exact match to the target budget, we select the configuration with the closest communication cost that is no smaller than that of \textbf{ICS}. Thus, \textbf{FedLite} is allowed to use an equal or slightly larger communication budget, making the comparison conservative in favor of the baseline.
\begin{itemize}
    \item \textbf{TS}~\cite{stich2018sparsified, zhou2024mask}: Each client selects and transmits only the elements of the feature with the largest magnitudes. Specifically, after forwarding the client-side model, clients identify the top-$N$ elements with the highest absolute values and sparsify the rest. 
    \item \textbf{RS}~\cite{stich2018sparsified, zhang2024federated}: Each client selects and transmits only the elements of the feature randomly without considering their magnitudes or positions within the vectors. 
    \item \textbf{RTS}~\cite{zheng2023reducing}: Each client selects and transmits a subset of the features by combining deterministic top-$N$ selection with random sampling. Specifically, after forwarding the client-side model, clients first identify top-$N$ elements with the highest absolute values and then, with a controlled probability, randomly select a few additional elements from the remaining elements. 
    \item \textbf{FedLite}~\cite{wang2022fedlite}: Each client compresses intermediate features by clustering similar features and transmitting the corresponding codebook and codewords to the server. In our implementation, we use its activation-clustering-based compression component as a baseline, without applying the additional gradient correction scheme, in order to focus on uplink intermediate feature compression and for fair comparisons. 
    \item \textbf{SplitFC}~\cite{oh2025communication}: Each client performs standard-deviation-based adaptive feature sparsification. Specifically, clients compute the standard-deviation of each intermediate feature vector and assign higher retention probabilities to feature vectors with larger standard-deviations. We do not apply the additional adaptive quantization, so that the comparison focuses on feature sparsification under the same sparsification ratio for fair comparisons.
\end{itemize}
Note that unlike \textbf{ICS}, all these baselines except \textbf{RS} incur additional client-side processing before transmission, such as magnitude ranking (\textbf{TS}, \textbf{RTS}), clustering and codebook generation (\textbf{FedLite}), or statistics computation (\textbf{SplitFC}).

\textbf{Comparison of Test Accuracy under i.i.d. Datasets.}
\begin{figure}[t]
    \centering
    \includegraphics[width=0.7\linewidth]{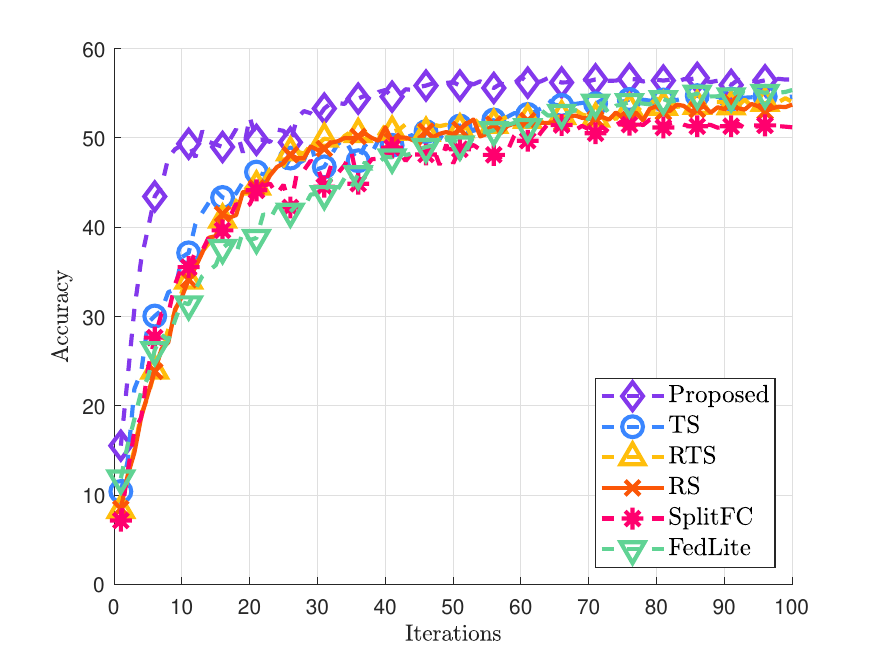}
    \caption{Test accuracy of CIFAR-100 classification under i.i.d. datasets.}
    \label{fig: cifar100_IID}
\end{figure}
We evaluate the performance of the proposed \textbf{ICS} framework on the CIFAR-100 dataset and compare it with baseline methods. As shown in Fig.~\ref{fig: cifar100_IID}, \textbf{ICS} consistently outperforms all baseline methods. This superior performance can be attributed to its capability to effectively capture and retain intermediate features that are intrinsically significant to  SL model’s predictive performance. Interestingly, \textbf{TS} demonstrates only marginal performance gains compared to \textbf{RS} and \textbf{RTS}. This outcome highlights a key distinction between SL and traditional FL. Specifically, it indicates that merely selecting intermediate feature elements based on magnitude alone does not translate effectively into improved performance in SL. 

It is also worth noting that \textbf{ICS} improves test accuracy faster than the baselines even in the early stage of training. This is notable because Grad-CAM is conventionally applied as a post-hoc tool for converged models. In contrast, \textbf{ICS} uses the true-class gradient with respect to the intermediate feature as a training-time importance metric, computed by the server during backpropagation at every iteration. It therefore quantifies the sensitivity of the current true-class score to each feature channel, and is updated as training progresses rather than relying on a fixed importance pattern. The early-stage gain suggests that this metric provides a useful task-sensitive selection criterion even when training starts from random initialization.
\begin{figure}[t]
    \centering
    \includegraphics[width=0.7\linewidth]{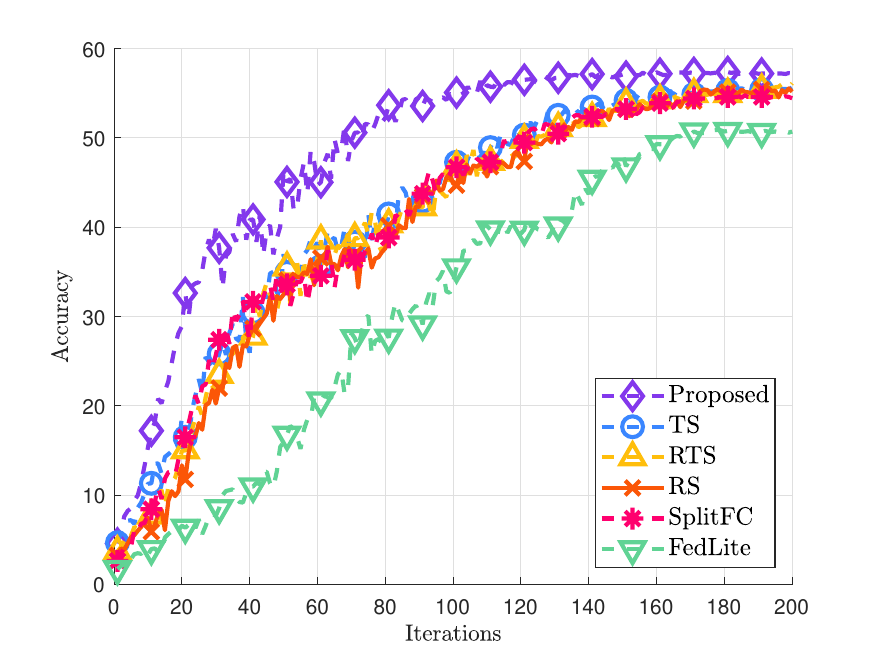}
    \caption{Test accuracy of CIFAR-100 classification with Dirichlet $\alpha=0.5$.}
    \label{fig: cifar100_NONIID}
\end{figure}
\begin{figure}[t]
    \centering
    \includegraphics[width=0.7\linewidth]{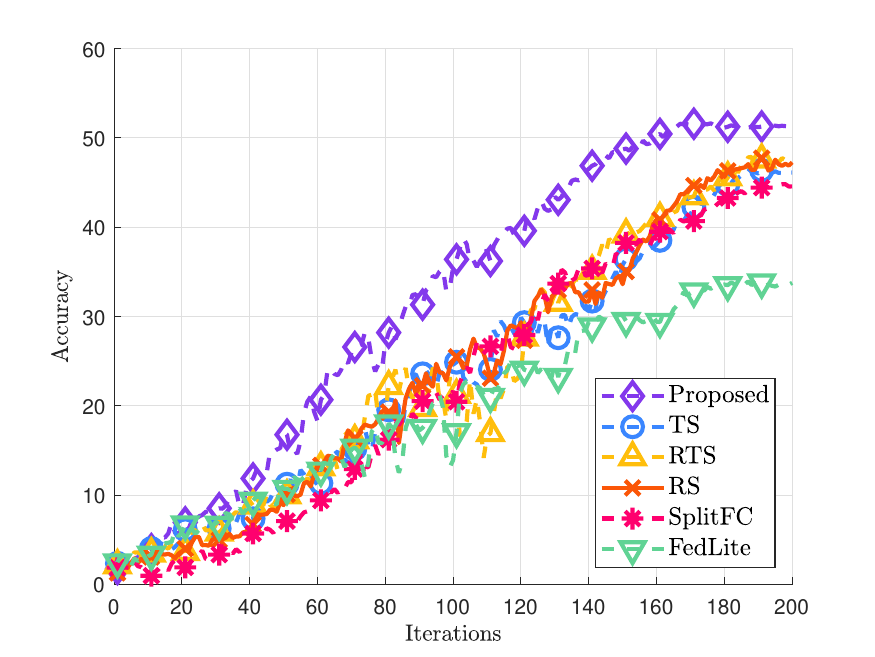}
    \caption{Test accuracy of CIFAR-100 classification with Dirichlet $\alpha=0.1$.}
    \label{fig: cifar100_alpha01}
\end{figure}
\begin{figure}[t]
    \centering
    \includegraphics[width=0.7\linewidth]{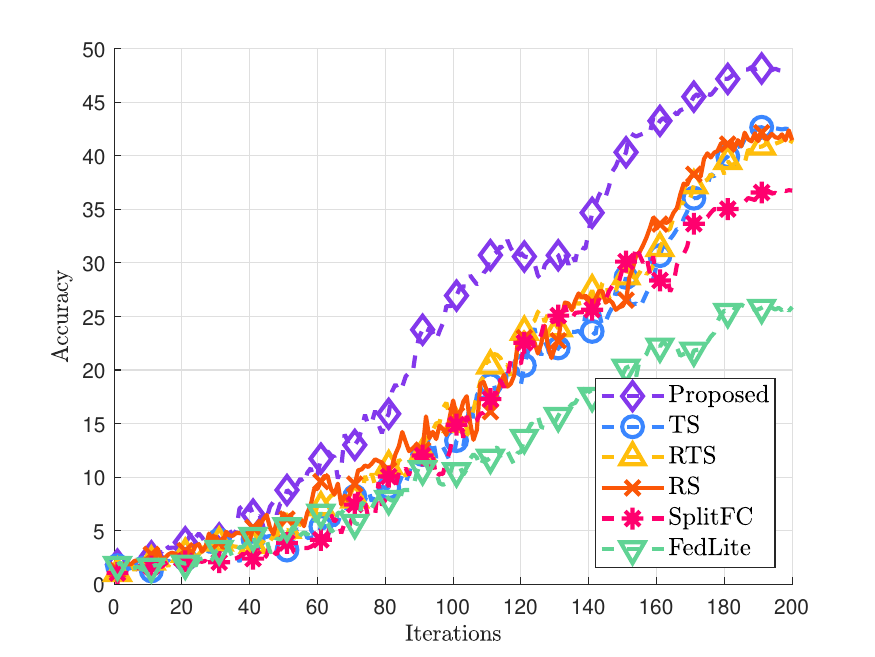}
    \caption{Test accuracy of CIFAR-100 classification with Dirichlet $\alpha=0.05$.}
    \label{fig: cifar100_alpha005}
\end{figure}

\textbf{Comparison of Test Accuracy under Non-i.i.d. Datasets.}
We further evaluate the performance of our proposed \textbf{ICS} framework under non-i.i.d. data distribution settings. In this experiment, datasets are generated using the Dirichlet distribution with concentration parameter $\alpha=0.5$. Fig. \ref{fig: cifar100_NONIID} presents the test accuracy of various sparsification strategies at the same sparsification ratio. Due to the inherently higher learning difficulty under non-i.i.d. conditions, the model requires more training iterations to converge, specifically, 200 iterations are executed. Fig. \ref{fig: cifar100_NONIID} shows that \textbf{ICS} consistently outperforms all baseline methods, demonstrating both superior final accuracy and notably faster convergence. 
Compared to the i.i.d. scenario, our approach achieves more pronounced advantages in terms of convergence speed under non-i.i.d. conditions. This is primarily because \textbf{ICS} selectively transmits intermediate features universally important across all true labels, whereas baseline methods perform sparsification without considering such universal relevance, thus limiting their convergence efficiency in non-i.i.d. environments.

\textbf{Effects of Non-i.i.d. Levels.}
\label{subsec: class_balance_ablation}
To further evaluate the robustness of the proposed \textbf{ICS} under varying degrees of data heterogeneity, we conduct additional experiments by changing the Dirichlet concentration parameter $\alpha$ used for client data partitioning. Specifically, we consider $\alpha \in \{0.05,0.1\}$, which induces more severe non-i.i.d. data partitions than the $\alpha=0.5$ setting considered in Fig.~\ref{fig: cifar100_NONIID}. In Figs.~\ref{fig: cifar100_alpha01} and~\ref{fig: cifar100_alpha005}, we present the corresponding test accuracy of \textbf{ICS} and the baseline methods.

As $\alpha$ decreases, all methods generally exhibit performance degradation due to the increasing discrepancy among client data distributions. However, \textbf{ICS} consistently outperforms the baselines, and the performance gap becomes more pronounced as the data distribution becomes more heterogeneous. In particular, the performance gap is most evident when $\alpha=0.05$, where the client data distributions become highly label-skewed and baseline sparsification methods are more likely to select features biased toward locally dominant classes. The advantage of \textbf{ICS} is closely related to its class-balanced importance aggregation. Unlike the baseline methods, \textbf{ICS} estimates class-specific importance vectors and aggregates them in a class-balanced manner, thereby reducing the influence of locally dominant labels on the sparsification policy. This allows \textbf{ICS} to better preserve feature channels that are discriminative across classes and maintain its effectiveness even under severe non-i.i.d. data distributions.

\begin{figure}[t]
    \centering
    \includegraphics[width=0.7\linewidth]{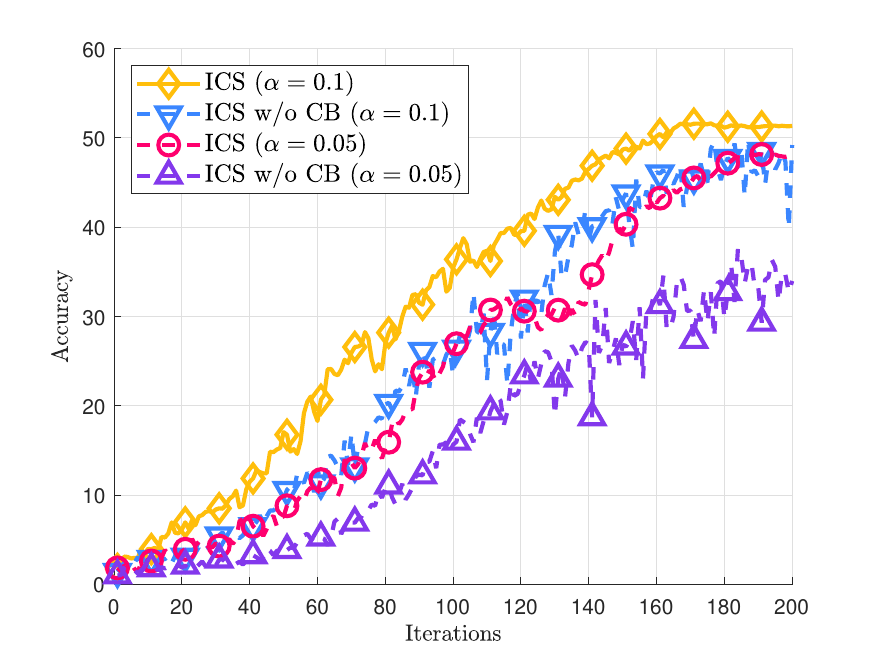}
    \caption{Test accuracy of \textbf{ICS} and \textbf{ICS w/o CB} on CIFAR-100 with Dirichlet $\alpha=0.1$ and $\alpha=0.05$.}
    \label{fig: cifar100_CB}
\end{figure}

\textbf{Effects of Class-Balanced Importance Aggregation.}
\label{subsec: ddd}
We further examine the contribution of the class-balanced aggregation criterion in \textbf{ICS}. For this purpose, we compare \textbf{ICS} with \textbf{ICS w/o CB} \cite{kim2025feature}, where CB denotes class balance. In \textbf{ICS w/o CB}, the importance vector is obtained by directly averaging the sample-wise importance vectors within each mini-batch, while all other procedures are kept identical to those of \textbf{ICS}.

Fig.~\ref{fig: cifar100_CB} shows the test accuracy under the same experimental setting as Fig.~\ref{fig: cifar100_NONIID}. To better reveal the effect of class-balanced aggregation, we consider two severe non-i.i.d. settings with $\alpha=0.1$ and $\alpha=0.05$. The results show that \textbf{ICS} consistently outperforms \textbf{ICS w/o CB} in both cases. In addition, \textbf{ICS} exhibits more stable training behavior, whereas \textbf{ICS w/o CB} shows larger accuracy fluctuations, particularly when $\alpha=0.05$. These results show that class-balanced aggregation mitigates the effect of mini-batch label imbalance when constructing the importance vector, leading to higher accuracy and more stable training under severe non-i.i.d. data heterogeneity.

\textbf{Effects of Mini-batch Size and Sparsification Ratio.}
Fig. \ref{fig: cifar100_BATCH} represents the test accuracy of CIFAR-100 with various mini-batch size and sparsification ratio under a fixed communication budget. In SL, each client transmits intermediate features of shape $B \times C \times U \times V$ to the server, where $B$ denotes the local mini‐batch size and $C,U,V$ are the feature channel, height, and width dimensions of the layer feature tensor, respectively. We fix the total number of transmitted elements per iteration and vary $B\in\{128,32,16\}$ alongside the sparsification ratio $R=\frac{N}{C}\in\{0.1,0.4,0.8\}$, respectively. 

Our results reveal an inherent trade-off between mini-batch size $B$ and sparsification ratio $R$ in communication-constrained SL, as predicted from Theorem \ref{theorem: convergence}. Specifically, increasing the mini-batch size leads to a higher total number of transmitted data samples. However, due to the lower $R$, each individual sample conveys less information, potentially impeding the training process. Conversely, while a smaller $B$ with a higher $R$ transmits fewer samples overall, it may still hinder convergence despite each sample carrying more information. Moreover, we find that higher $B$ setups yield more rapid increases in test accuracy during the first few epochs, indicating that exposure to a larger variety of examples outweighs per-sample fidelity at the outset. As training progresses, the advantage shifts toward higher $R$, which reduce the per-epoch accuracy gap by preserving fuller feature information after broad sample exposure. 

\begin{figure}[t]
    \centering
    \includegraphics[width=0.7\linewidth]{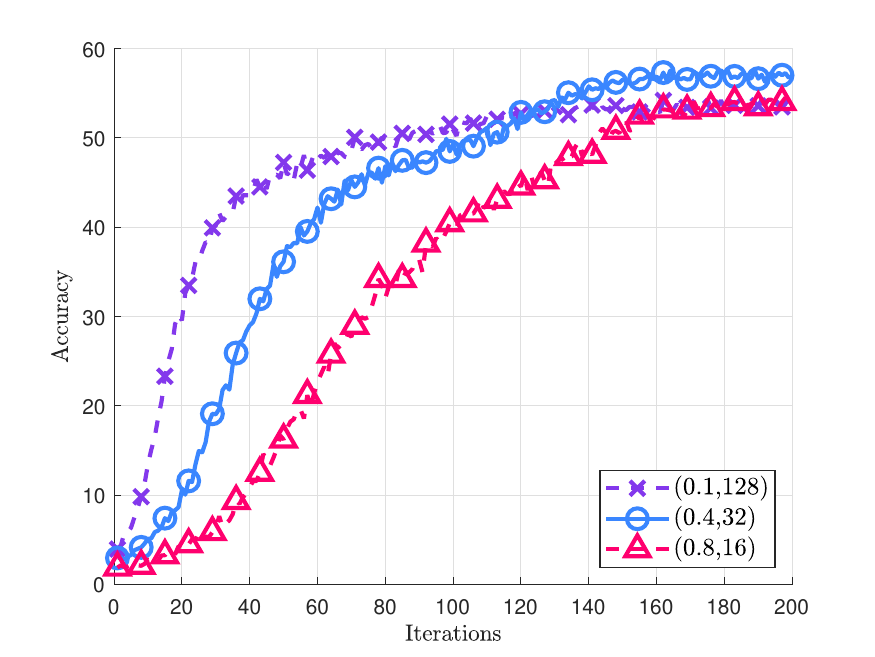}
    \caption{Test accuracy of CIFAR-100 classification under different sparsification ratio and mini-batch size pair.}
    \label{fig: cifar100_BATCH}
\end{figure}
\textbf{Evaluation under PSL.}
\begin{figure}[t]
    \centering
    \includegraphics[width=0.7\linewidth]{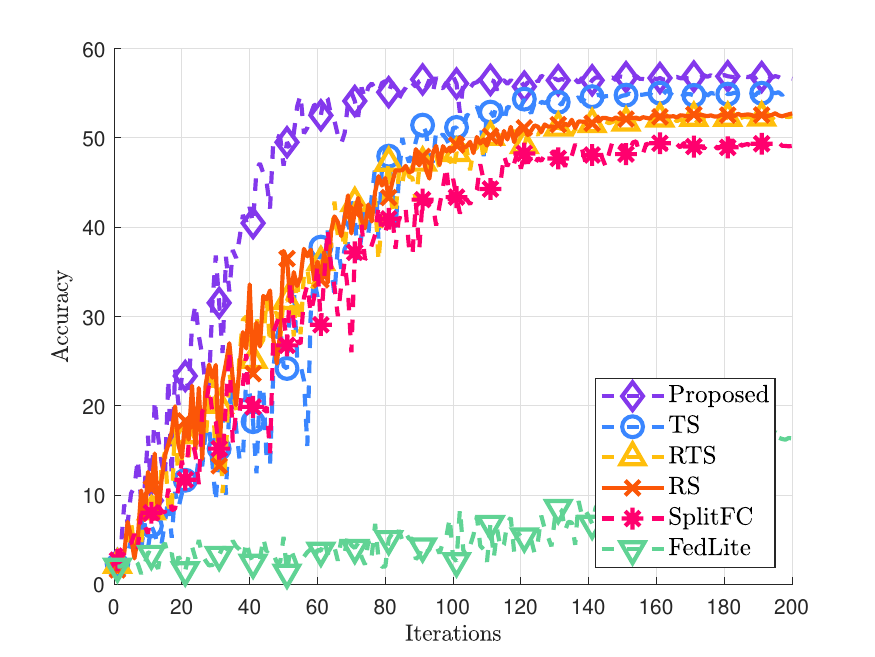}
    \caption{Test accuracy on CIFAR-100 classification under PSL with Dirichlet $\alpha=0.5$.}
    \label{fig: PSL}
\end{figure}
We further evaluate the proposed \textbf{ICS} framework under the PSL setting described in Section~\ref{sec:psl_extension}. Fig.~\ref{fig: PSL} shows the test accuracy under the PSL setting with Dirichlet $\alpha=0.5$ and mini-batches of size $32$. The proposed \textbf{ICS} achieves the best performance among all the compared methods, demonstrating that the shared importance vector remains effective even when multiple clients participate in parallel.

An interesting observation is that \textbf{FedLite} exhibits a larger performance degradation than that in the sequential SL setting. This may be due to the interaction between its data-dependent codebook construction and client-side model aggregation. Under PSL, each client performs local updates in parallel using features generated from its own non-i.i.d. data, and the resulting client-side models are then aggregated into a single model for the next iteration. Since each client compresses its features with an independently estimated codebook, the local updates of different clients are shaped by client-dependent quantization errors. When these models are aggregated, the differing quantization errors may not be canceled and instead mix into a single model, which can yield less coherent updates. This effect recurs at every aggregation step and is more pronounced under the non-i.i.d. partition with $\alpha=0.5$, where the parallel client-side models can diverge more strongly. In contrast, \textbf{ICS} uses a shared server-generated importance vector, allowing all clients to sparsify their intermediate features according to the same task-aware criterion. This shared sparsification rule helps maintain consistency across parallel client updates, making \textbf{ICS} more compatible with PSL.

\textbf{Generalization to Transformer-based SL Models.}
\begin{figure}[t]
    \centering
    \includegraphics[width=0.7\linewidth]{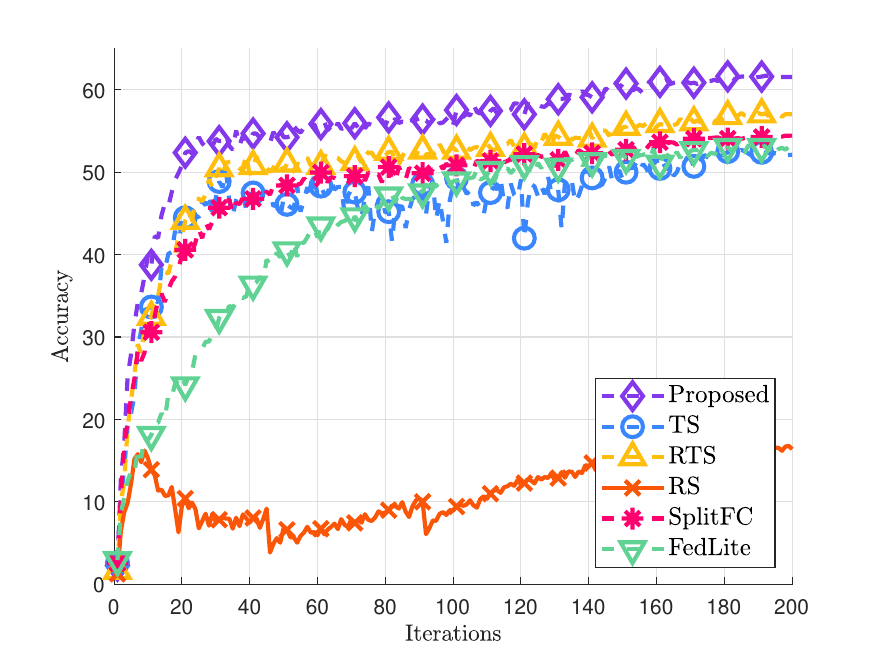}
    \caption{Test accuracy of CIFAR-100 classification with transformer-based SL.}
    \label{fig: Transformer_IID}
\end{figure}
To verify that \textbf{ICS} is not limited to CNN-based architectures, we further evaluate it with a transformer-based SL model. Following Remark~\ref{remark: architecture_agnostic}, we evaluate \textbf{ICS} on CIFAR-100 using a split transformer model, where sparsification is applied along the embedding dimension $d_{EMB}$ of the token embedding tensor. The model has eight transformer encoder layers with embedding dimension $256$, eight attention heads, MLP ratio $4$, and dropout $0.1$. The first two encoder layers are placed on the client side, while the remaining six encoder layers, followed by layer normalization, mean pooling over patch tokens, and the classification head, are placed on the server side. The experiment is conducted with 10 clients under i.i.d. partitioning using AdamW with learning rate $3 \times 10^{-3}$, weight decay $10^{-2}$ and cosine learning-rate decay. As shown in Fig.~\ref{fig: Transformer_IID}, \textbf{ICS} achieves stable training performance, whereas \textbf{RS} fails to train reliably under the same communication budget. This result verifies that \textbf{ICS} can be extended to transformer-based SL by applying feature-wise sparsification along the embedding dimension.

\begin{figure}[t]
    \centering
    \includegraphics[width=0.7\linewidth]{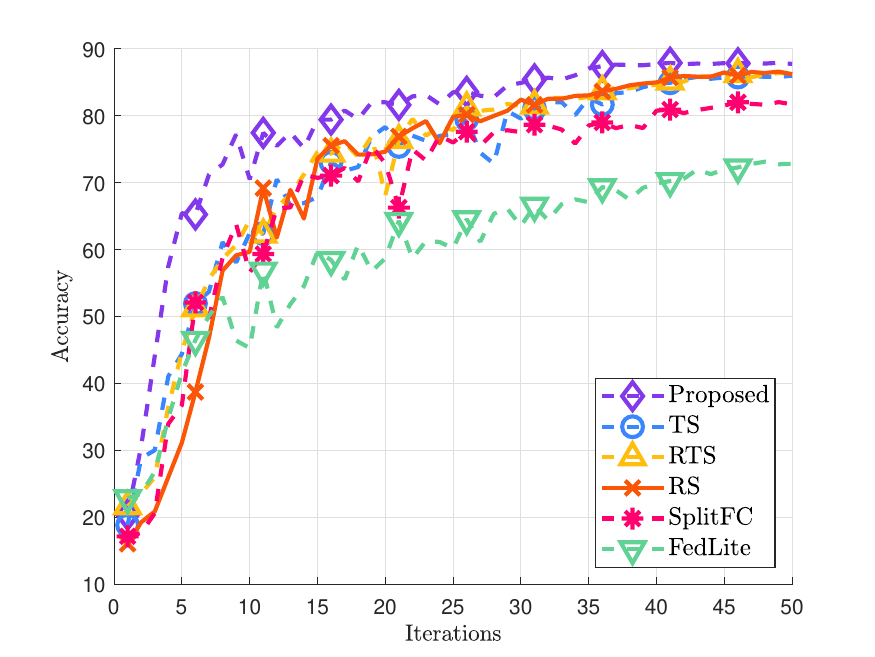}
    \caption{Test accuracy of CIFAR-10 classification under i.i.d. datasets.}
    \label{fig: cifar10_IID}
\end{figure}
\begin{figure}[t]
    \centering
    \includegraphics[width=0.7\linewidth]{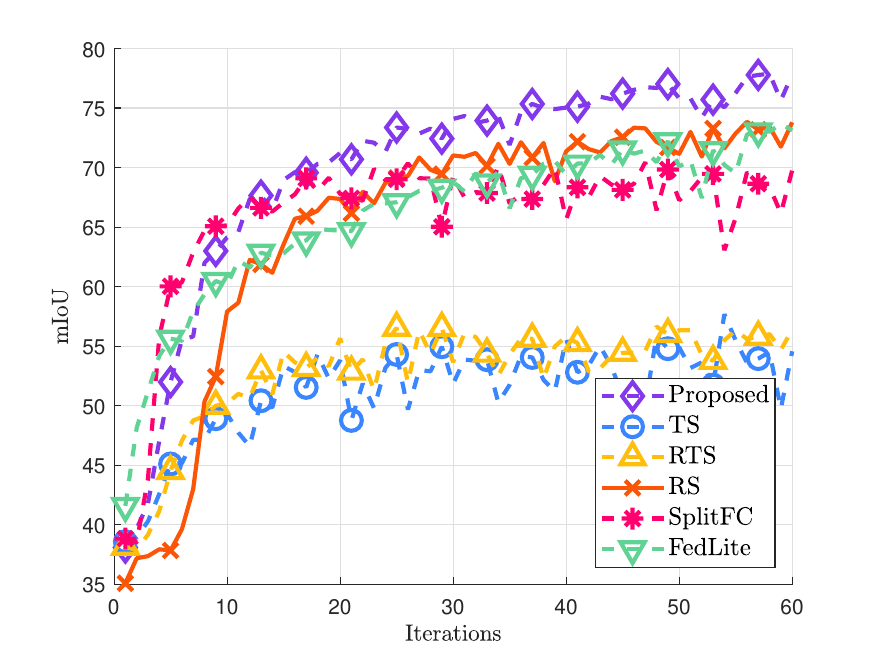}
    \caption{mIoU of segmentation under Oxford-IIIT Pet datasets.}
    \label{fig: seg_MIOU}
\end{figure}
\begin{figure}[t]
    \centering
    \includegraphics[width=0.7\linewidth]{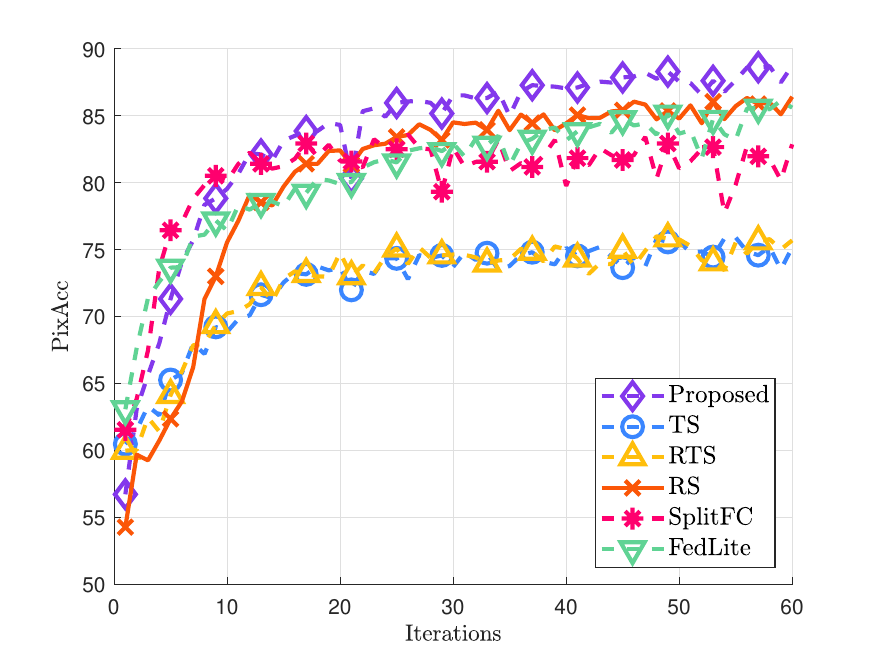}
    \caption{Pixel accuracy of segmentation under Oxford-IIIT Pet datasets.}
    \label{fig: seg_PIXACC}
\end{figure}
\textbf{Generalization to a Different Dataset.} To assess the generalizability of our approach across different datasets, we conduct additional experiments on CIFAR-10. The experimental setup was identical to that used for CIFAR-100, except for a fixed mini-batch size of $256$, a learning rate of $0.1$ and a total of $50$ iterations. As illustrated in Fig. \ref{fig: cifar10_IID}, $\textbf{ICS}$ consistently outperforms alternative sparsification schemes. These results underscore the importance of effectively preserving critical intermediate feature information and demonstrate the broad applicability of our approach across multiple datasets.

\begin{figure*}[t]
    \centering
    \includegraphics[width=0.98\textwidth]{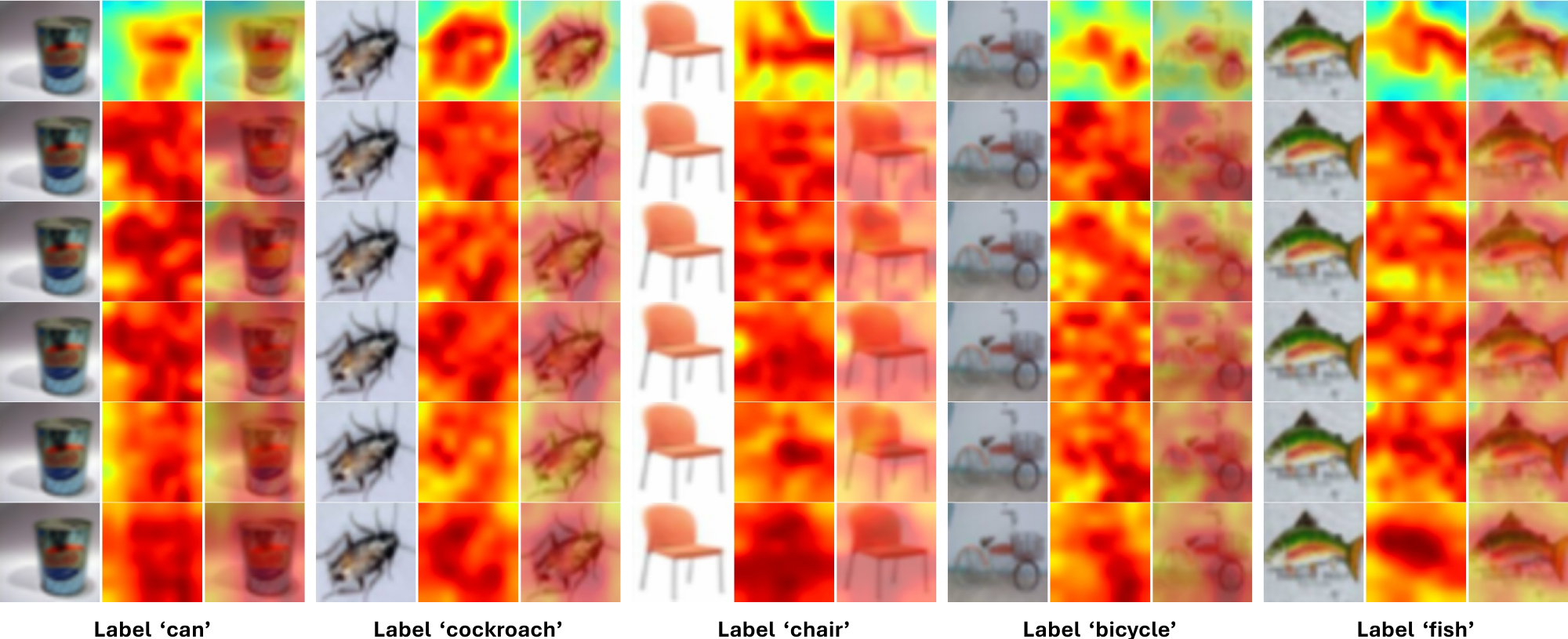}
    \caption{Comparison of feature saliency maps for sampled images from five different labels. For each sample, the first, second, and third columns show the original image, the Grad-CAM saliency map, and their overlay, respectively. Each row corresponds to a different sparsification method: \textbf{ICS} (row 1), \textbf{TS} (row 2), \textbf{RS} (row 3), \textbf{RTS} (row 4), \textbf{SplitFC} (row 5), and \textbf{FedLite} (row 6).}\label{fig: gradcam}
\end{figure*}
\textbf{Generalization to a Different Task.}
We further validate the generalizability of \textbf{ICS} by extending our evaluation to a segmentation task using the Oxford-IIIT Pet dataset \cite{parkhi2012cats}. This task requires the model to produce pixel-level predictions and is thus highly sensitive to the quality of spatial information. Both the client and the server were trained using AdamW, allowing us to examine whether the relative behavior of the methods persists under an adaptive optimizer. Figs. \ref{fig: seg_MIOU} and  \ref{fig: seg_PIXACC} illustrate the validation of mean intersection over union (mIoU) and pixel accuracy, respectively. The results reveal a significant shift in the relative performance of the baseline methods compared to the classification tasks. While \textbf{ICS} consistently achieves the highest mIoU and accuracy, \textbf{TS} and \textbf{RTS} show a clear performance degradation compared with the other methods. This trend differs from the classification results, where the three baselines exhibited comparable performance.

This performance gap can be attributed to the inherent characteristics of semantic segmentation. In such datasets, background regions typically occupy a dominant portion of the image and are easier to recognize than object regions, which are often small and spatially localized. This object and background imbalance can bias magnitude-based feature selection schemes to predominately capture background patterns. Accordingly, \textbf{TS} and \textbf{RTS} are more likely to retain background dominant feature channels while discarding object dominant feature channels, leading to degraded performance. This observation is consistent with prior studies reporting that semantic segmentation datasets exhibit object and background imbalance, where background pixels overshadow object class statistics and suppress gradients for object regions \cite{tappeiner2022tackling}.

In contrast, \textbf{RS} and \textbf{SplitFC} do not rely on magnitude and therefore avoid such background-driven bias. Their stochastic selection introduces a regularization effect, similar in spirit to feature channel-wise dropout. This likely encourages the encoder to learn redundant and diverse representations for the object classes, distributing information across multiple feature channels.  While \textbf{RS} and \textbf{SplitFC} demonstrate strong robustness, their efficacy relies on redundancy learned through stochastic selection. \textbf{ICS}, in contrast, achieves consistently superior accuracy by selecting feature channels based on their server-calculated gradient importance. Crucially, its class-balanced aggregation mechanism directly mitigates the background dominance that is an inherent limitation of \textbf{TS}. By applying equal weighting to the importance scores of the object classes and the background class, \textbf{ICS} prevents the importance vector from collapsing onto background-dominant feature channels. These results suggest that \textbf{ICS} performs well not only on image-level classification tasks but also on pixel-level segmentation tasks.

\textbf{Ablation Study.}
To provide a deeper qualitative understanding of why our proposed \textbf{ICS} achieves superior performance, we visualize the feature-level importance retained by each sparsification method. For this analysis, we adopt the same experimental setup used for CIFAR-100 i.i.d. scenarios and employ Grad-CAM \cite{selvaraju2017grad} to generate saliency maps derived from the transmitted feature channels.

Fig. \ref{fig: gradcam} presents these visualizations for representative test samples from the CIFAR-100 dataset, including ground-truth labels. Within each subfigure, the columns display the original image, the resulting saliency heatmap, and an overlay of the two, respectively. The rows correspond to the different sparsification methods being compared, specifically, the rows represent \textbf{ICS}, \textbf{TS}, \textbf{RS}, \textbf{RTS}, \textbf{SplitFC}, and \textbf{FedLite}, respectively. As shown across all subfigures, heatmaps of $\textbf{ICS}$ are localized and concentrate almost exclusively on the meaningful regions of the objects. This indicates that our importance-aware mechanism effectively identifies and prioritizes the most discriminative feature channels for the task.

In contrast, the saliency map for \textbf{RS} exhibits no discernible spatial structure, as expected given its random and feature-agnostic selection strategy. Similarly, the maps for \textbf{TS} and \textbf{RTS} appear consistently diffuse, failing to localize the target objects and instead spreading attention across irrelevant background regions. The maps for \textbf{SplitFC} and \textbf{FedLite} also show limited object localization. This is because \textbf{SplitFC} selects features based on feature-level statistics rather than their direct contribution to the task objective, while \textbf{FedLite} compresses intermediate features through clustering-based codebook construction without explicitly preserving task-discriminative feature channels. These results visually reinforce our core hypothesis that feature magnitude, feature statistics, and clustering-based representation compression is an unreliable indicator of feature importance. The qualitative evidence thus complements our quantitative findings, explaining why \textbf{ICS} achieves superior accuracy and convergence by transmitting only the most salient and informative features. 

\section{Conclusion}
This paper introduced ICS, an importance-aware class-balanced sparsification method for wireless split learning. The server estimates Grad-CAM-based channel importance from the true-class logit during backpropagation and aggregates it into a class-balanced, label-agnostic vector, which each client reuses in the next round to transmit only the most discriminative feature channels without any additional client-side computation. The class-balanced design mitigates the bias that arises under non-i.i.d. client data, and the convergence analysis reveals a trade-off between the mini-batch size and the sparsification ratio under a fixed communication budget. We further showed that ICS extends to parallel split learning through a shared server-generated importance vector and to transformer-based models along the embedding dimension. Experiments across classification and semantic segmentation, several heterogeneity levels, and different architectures confirmed that ICS consistently outperforms existing communication-efficient SL baselines while keeping the client-side overhead minimal.

\appendices
\section{Proof of Theorem \ref{theorem: convergence}}\label{apdx of theorem: convergence}
To begin with, we first recall the global model, i.e. $\theta^t=[\theta_{d}^t||\theta_{s}^t]$. The model update can be represented as
\begin{align}
    \theta^{t+1}=\theta^t-\sum_{k=0}^{K-1}\sum_{e=0}^{E-1} \eta^t\nabla F_{k}(\tilde{z}_k^{t,e},\mathcal{B}^{t,e}_k;\theta_{k}^{t,e}),
\end{align}
where $\eta^t$ is learning rate at iteration $t$, $E$ is the number of local steps, $\nabla F_{k}(\tilde{z}_k^{t,e},\mathcal{B}^{t,e}_k;\theta_{k}^{t,e})$ represents the concatenated stochastic gradients with sparsification of client $k$ and the server,  i.e. $\nabla F_{k}(\tilde{z}_k^{t,e},\mathcal{B}^{t,e}_k;\theta_{k}^{t,e})=[\nabla_{\theta_{d,k}^{t,e}} F_{k}(\tilde{z}_k^{t,e},\mathcal{B}^{t,e}_k;\theta_{k}^{t,e})||\nabla_{\theta_{s,k}^{t,e}} F_{k}(\tilde{z}_k^{t,e},\mathcal{B}^{t,e}_k;\theta_{k}^{t,e})]$.

Next, using $S$-smoothness, the following inequality holds:
\begin{align}\label{eq: 1-round convergence}
    \mathbb{E}[F(\theta^{t+1})-F(\theta^t)]\leq\underbrace{\mathbb{E}[\nabla F(\theta^t)^\textsf{T}(\theta^{t+1}-\theta^t)]}_{(a)}\nonumber\\+\frac{S}{2}\underbrace{\mathbb{E}[\|\theta^{t+1}-\theta^t\|_2^2]}_{(b)}. 
\end{align}
Now we can derive an upper bound of $(a)$ in \eqref{eq: 1-round convergence} as follows:
\begin{align}
    (a)&=-\eta^t\sum_{k=0}^{K-1}\sum_{e=0}^{E-1}\mathbb{E}\left[\nabla F(\theta^t)^\textsf{T}\nabla F_k(\theta_{k}^{t,e})\right]\\
    &= -\frac{\eta^tKE}{2}\mathbb{E}\left[\left\|\nabla F(\theta^t)\right\|_2^2\right]\nonumber\\
    &\quad-\frac{\eta^t}{2}\sum_{k=0}^{K-1}\sum_{e=0}^{E-1}\mathbb{E}\left[\left\|\nabla F_k(\theta_{k}^{t,e})\right\|_2^2\right]\nonumber\\
    &\quad+\frac{\eta^t}{2}\underbrace{\sum_{k=0}^{K-1}\sum_{e=0}^{E-1}\mathbb{E}\left[\left\|\nabla F(\theta^t)-\nabla F_k(\theta_{k}^{t,e})\right\|_2^2\right]}_{(a.1)},\label{eq: split0}
\end{align}
where the first equality is due to \textbf{Assumption \ref{as:local gradient unbiasedness}} and the second equality comes from a simple property, such as $-\mathbf{a}^{T}\mathbf{b} = \frac{1}{2}(-\|\mathbf{a}\|_2^2 - \|\mathbf{b}\|_2^2 + \|\mathbf{a - b}\|_2^2)$. For $(a.1)$ in \eqref{eq: split0}, by adding, subtracting $\nabla F_k(\theta^t)$ and using  $\|\mathbf{a}+\mathbf{b}\|_2^2\leq2\|\mathbf{a}\|_2^2+2\|\mathbf{b}\|_2^2$, we have
\begin{align}
    (a.1)&\leq\sum_{k=0}^{K-1}\sum_{e=0}^{E-1}\mathbb{E}\Bigg[2\left\|\nabla F(\theta^t)-\nabla F_k(\theta^{t})\right\|_2^2\nonumber\\
    &\qquad\qquad\qquad+2\left\|\nabla F_k(\theta^t)-\nabla F_k(\theta_{k}^{t,e})\right\|_2^2\Bigg]\label{eq: split1}\\
    &\leq 2KE\psi^2+2\underbrace{\sum_{k=0}^{K-1}\sum_{e=0}^{E-1}\mathbb{E}\left[\left\|\nabla F_k(\theta^t)-\nabla F_k(\theta_{k}^{t,e})\right\|_2^2\right]}_{(a.1.1)}, \label{eq: a.1.1}
\end{align}
where the second inequality comes from $\textbf{Assumption \ref{as: client heterogeneity}}$.

For $(a.1.1)$ in \eqref{eq: a.1.1}, using $S$-smoothness, we have
\begin{align}
    &(a.1.1)\leq S^2\sum_{k=0}^{K-1}\sum_{e=0}^{E-1}\mathbb{E}\left[\left\|\theta^{t,0}_{0}-\theta_{k}^{t,e}\right\|_2^2\right]\\
    &=S^2\sum_{k=0}^{K-1}\sum_{e=0}^{E-1}\mathbb{E}\Bigg[\Big\|\sum_{i=0}^{k-1}\sum_{j=0}^{E-1}\eta^t\nabla F_i(\mathcal{B}_i^{t,j};\theta_{i}^{t,j})\nonumber\\
    &\qquad\qquad\qquad\qquad+\sum_{h=0}^{e-1}\eta^t\nabla F_k(\mathcal{B}_k^{t,h};\theta_{k}^{t,h})\Big\|_2^2\Bigg]\label{eq: split_2}\\
    &\leq S^2\sum_{k=0}^{K-1}\sum_{e=0}^{E-1}(kE+e)\Bigg(\sum_{i=0}^{k-1}\sum_{j=0}^{E-1}\mathbb{E}\left[\|\eta^t\nabla F_i(\mathcal{B}_i^{t,j};\theta_{i}^{t,j})\|_2^2\right]\nonumber\\
    &\qquad\qquad\qquad\qquad+\sum_{h=0}^{e-1}\mathbb{E}\left[\|\eta^t\nabla F_k(\mathcal{B}_k^{t,h};\theta_{k}^{t,h})\|_2^2\right]\Bigg)\label{eq: split_3}\\
    &=S^2\sum_{k=0}^{K-1}\sum_{e=0}^{E-1}(kE+e)(\eta^t)^2\Bigg(\sum_{i=0}^{k-1}\sum_{j=0}^{E-1}\mathbb{E}\Big[\|\nabla F_i(\theta_{i}^{t,j})\|_2^2\nonumber\\
    &+\|\nabla F_i(\mathcal{B}_i^{t,j};\theta_{i}^{t,j})-\nabla F_i(\theta_{i}^{t,j})\|_2^2\Big]+\sum_{h=0}^{e-1}\mathbb{E}\Big[\|\nabla F_k(\theta_{k}^{t,h})\|_2^2\nonumber\\
    &+\|\nabla F_k(\mathcal{B}_k^{t,h};\theta_{k}^{t,h})-\nabla F_k(\theta_{k}^{t,h})\|_2^2\Big]\Bigg)\label{eq: split_4}\\
    &\leq S^2(\eta^t)^2\sum_{k=0}^{K-1}\sum_{e=0}^{E-1}(kE+e)^2\left(\frac{\sigma^2}{B}+G^2\right)\\
    &= S^2(\eta^t)^2\left(\frac{\sigma^2}{B}+G^2\right)\frac{KE(KE-1)(2KE-1)}{6}\\
    &\leq \frac{1}{3}S^2(\eta^t)^2\left(\frac{\sigma^2}{B}+G^2\right)(KE)^3,
\end{align}
where the first equality is obtained by recursively applying the sequential SL update rule up to the beginning of the $e$-th local step of client $k$. Specifically, before the $e$-th local update of client $k$, the model has already been updated by all $E$ local steps of clients $0,\ldots,k-1$ and by the first $e$ local steps of client $k$, starting from $\theta_0^{t,0}$, which gives
\begin{align}\label{eq: theta_expand}
    \theta_{k}^{t,e}=\theta_{0}^{t,0}-\eta^t\Bigg(\sum_{i=0}^{k-1}&\sum_{j=0}^{E-1}\nabla F_i(\mathcal{B}_i^{t,j};\theta_{i}^{t,j})\nonumber\\
    &+\sum_{h=0}^{e-1}\nabla F_k(\mathcal{B}_k^{t,h};\theta_{k}^{t,h})\Bigg).
\end{align}
Substituting this expansion into $\theta_0^{t,0}-\theta_k^{t,e}$ yields the first equality. The second inequality is due to Cauchy–Schwarz inequality, and the second equality is obtained by adding and subtracting $\nabla F_i(\theta_i^{t,j})$ and $\nabla F_k(\theta_k^{t,h})$. The cross terms in the equality go to zero due to the unbiasedness. The third inequality comes from \textbf{Assumptions \ref{as:local gradient unbiasedness}} and \textbf{\ref{as: norm boundedness}}. The last inequality is due to $KE-1\leq KE$ and $\;2KE-1\leq2KE$.

Before upper-bound $(b)$ in \eqref{eq: 1-round convergence}, we first introduce the following \textbf{Lemma \ref{lemma: sparsification error}} for upper-bound $(b)$.
\begin{lemma}\label{lemma: sparsification error}
    Under \textbf{Assumptions \ref{as: sparsification}}, \textbf{\ref{as: norm boundedness}}, and let
    \begin{align}
        m_k^{t,e}=\nabla F_{k}(\tilde{z}_k^{t,e},\mathcal{B}^{t,e}_k;\theta_{k}^{t,e})-\nabla F_{k}(z_k^{t,e},\mathcal{B}^{t,e}_k;\theta_{k}^{t,e})
    \end{align}
    where $m_k^{t,e}=[m_{d,k}^{t,e}||m_{s,k}^{t,e}]$, be the update error caused by the sparsification operator $\mathcal{S}(\cdot)$. Then, $m_k^{t,e}$ is upper-bounded as
    \begin{align}
        \mathbb{E}\|m_k^{t,e}\|_2^2\leq\Lambda^2\left(1-\frac{N}{C}\right)B\delta^2,
    \end{align}
    where $\Lambda^2=\Lambda_0^2(\Lambda_1^2+\Lambda_2^2\Lambda_3^2)$. Here, $\Lambda_0$, $\Lambda_1$, $\Lambda_2$ and $\Lambda_3$ are spectral norm upper-bound of $\frac{\partial\mathcal{L}}{\partial \hat{y}_{k,b}^{t,e}}$, $\frac{\partial^2\hat{y}^{t,e}_{k,b,\ell}}{\partial z_k^{t,e}\partial\theta_{s,k}^{t,e}}$, $\frac{\partial^2 \hat{y}^{t,e}_{k,b,\ell}}{\partial \left(z_k^{t,e}\right)^2}$ and $\frac{\partial z_k^{t,e}}{\partial \theta_{d,k}^{t,e}}$ $\forall k,t,\ell$, respectively.
\end{lemma}
\begin{IEEEproof}
    See Appendix \ref{apdx of lemma: sparsification error}.
\end{IEEEproof}
Next, using Cauchy–Schwarz inequality, $(b)$ in \eqref{eq: 1-round convergence} can be expressed as
\begin{align}
    (b)&\leq KE(\eta^t)^2\sum_{k=0}^{K-1}\sum_{e=0}^{E-1}\mathbb{E}\left[\|\nabla F_{k}(\mathcal{B}^{t,e}_k;\theta_{k}^{t,e})+m_{k}^{t,e}\|_2^2\right]\nonumber\\
    &\leq 2KE(\eta^t)^2\underbrace{\sum_{k=0}^{K-1}\sum_{e=0}^{E-1}\mathbb{E}\left[\|\nabla F_{k}(\mathcal{B}^{t,e}_k;\theta_{k}^{t,e})\|_2^2\right]}_{(b.1)}\nonumber\\
    &\;\;\;\;+2KE(\eta^t)^2\underbrace{\sum_{k=0}^{K-1}\sum_{e=0}^{E-1}\mathbb{E}\left[\|m_{k}^{t,e}\|_2^2\right]}_{(b.2)}, \label{eq: split5}
\end{align}
where $z_k^{t,e}$ in $\nabla F_{k}(z_k^{t,e},\mathcal{B}^{t,e}_k;\theta_{k}^{t,e})$ is omitted for ease of notation. The first inequality follows from
$\|\sum_{r=1}^{KE}\mathbf a_r\|_2^2
\leq KE\sum_{r=1}^{KE}\|\mathbf a_r\|_2^2$ and the second inequality follows from $\|\mathbf a+\mathbf b\|_2^2
\leq2\|\mathbf a\|_2^2+2\|\mathbf b\|_2^2$.

By adding and subtracting $\nabla F_k(\theta_k^{t,e})$ in $(b.1)$, we can get
\begin{align}
    (b.1)&=\sum_{k=0}^{K-1}\sum_{e=0}^{E-1}\mathbb{E}\Big[\|\nabla F_{k}(\mathcal{B}^{t,e}_k;\theta_{k}^{t,e})-\nabla F_k(\theta_k^{t,e})\|_2^2\nonumber\\
    &\qquad\qquad\qquad+\|\nabla F_k(\theta_k^{t,e})\|_2^2\Big]\\
    &\leq \frac{KE\sigma^2}{B}+KEG^2,\label{eq: b.1}
\end{align}
where the cross term in the equality goes to zero due to the unbiasedness, i.e., $\mathbb{E}\left[
\nabla F_k(\mathcal B_k^{t,e};\theta_k^{t,e})\right]=\mathbb{E}\left[\nabla F_k(\theta_k^{t,e})
\right]$, and the inequality comes from \textbf{Assumptions \ref{as:local gradient unbiasedness}} and \textbf{\ref{as: norm boundedness}}. Finally, by using $\textbf{Lemma \ref{lemma: sparsification error}}$ for upper-bound $(b.2)$ in \eqref{eq: split5} and substitute $\eqref{eq: b.1}$ into $\eqref{eq: split5}$, $(b)$ in \eqref{eq: 1-round convergence} can be upper-bounded as
\begin{align}
    (b)\leq 2K^2E^2(\eta^t)^2\left(\frac{\sigma^2}{B}+G^2+\Lambda^2\left(1-\frac{N}{C}\right)B\delta^2\right).
\end{align}
Substituting the above inequalities to \eqref{eq: 1-round convergence}, we get
\begin{align}
    \mathbb{E}&[F(\theta^{t+1})-F(\theta^t)] \leq -\frac{\eta^tKE}{2}\mathbb{E}\left[\left\|\nabla F(\theta^t)\right\|_2^2\right]+\eta^tKE\psi^2\nonumber\\
    &+\frac{1}{3}S^2(\eta^t)^3\left(\frac{\sigma^2}{B}+G^2\right)(KE)^3+SK^2E^2(\eta^t)^2\Bigg(\frac{\sigma^2}{B}\nonumber\\
    &+G^2+\Lambda^2\left(1-\frac{N}{C}\right)B\delta^2\Bigg).
\end{align}
Averaging the above inequality over iteration from $0$ to $T - 1$ and applying $\eta^t=\frac{1}{KE\sqrt{T}}$, we can get
\begin{align}
    &\mathbb{E}\left[\frac{1}{T}\sum_{t=0}^{T-1}\|\nabla F(\theta^t)\|_2^2\right]\leq \frac{2}{\sqrt{T}}\mathbb{E}[F(\theta^{0})-F(\theta^{*})]+ 2\psi^2\nonumber\\
    &+ \frac{2S^2}{3T}\left(\frac{\sigma^2}{B}+G^2\right)
    +\frac{2S}{\sqrt{T}}\Bigg(\frac{\sigma^2}{B}+G^2+\Lambda^2\left(1-R\right)B\delta^2\Bigg), \label{eq: split6}
\end{align}
where the first term in the right-hand side is upper-bounded as $\mathbb{E}[F(\theta^{0})-F(\theta^{T})] \leq \mathbb{E}[F(\theta^{0})-F(\theta^{*})]$ since $F(\theta^{T}) \geq F(\theta^{*})$. 

\section{Proof of Lemma \ref{lemma: sparsification error}}\label{apdx of lemma: sparsification error}

We first begin with the server-side error. For ease of notation, we omit indices of $k$ and $e$. Similar to \cite{wang2022fedlite}, using the definition of the server-side gradient, the update error caused by sparsification can be expressed as 
\begin{align}
    \|m_s^t\|_2^2=\left\|\frac{1}{B}\sum_{b=0}^{B-1}\frac{\partial\mathcal{L}}{\partial \hat{y}_{b}^t}\left(\frac{\partial \hat{y}^{t}_{b,\ell}(\tilde{z}_k^t;\theta_s^t)}{\partial\theta_s^t}-\frac{\partial\hat{y}^{t}_{b,\ell}(z_k^t;\theta_s^t)}{\partial\theta_s^t}\right)\right\|_2^2\label{eq: sparsification}
    \end{align}
where $\tilde{z}^t$ and $z^t$ denote intermediate feature after sparsification and before sparsification, respectively. Next using the mean value theorem,
\begin{align}
    \|m_s^t\|_2^2&\leq \max_b\left\|\frac{\partial\mathcal{L}}{\partial \hat{y}_{b}^t}\frac{\partial^2\hat{y}^{t}_{b,\ell}(z_k^t;\theta_s^t)}{\partial z\partial\theta_s^t}(\tilde{z}^t-z^t)\right\|_2^2\\
    &\leq \Lambda_0^2\Lambda_1^2\left(1-\frac{N}{C}\right)B\delta^2
\end{align}
where the second inequality is due to \textbf{Assumptions \ref{as: sparsification}} and \textbf{\ref{as: norm boundedness}}. Here, $\Lambda_0$ and $\Lambda_1$ are spectral norm upper-bound of $\frac{\partial\mathcal{L}}{\partial \hat{y}_{b}^t}$ and $\frac{\partial^2\hat{y}^{t}_{b,\ell}(z^t;\theta_s^t)}{\partial z^t\partial\theta_s^t}$, respectively.

Similar to the server-side error, the update error caused by sparsification at the client-side can be expressed as 
\begin{align}
    \|m_d^t\|_2^2&=\Bigg\|\frac{1}{B}\sum_{b=0}^{B-1}\frac{\partial\mathcal{L}}{\partial \hat{y}_{b}^t}\Bigg(\frac{\partial \hat{y}^{t}_{b,\ell}(\tilde{z}^t;\theta_s^t)}{\partial\tilde{z}^t}\nonumber\\
    &\qquad\qquad\qquad\qquad\quad-\frac{\partial\hat{y}^{t}_{b,\ell}(z^t;\theta_s^t)}{\partial z^t}\Bigg)\frac{\partial z^t}{\partial\theta_d^t}\Bigg\|_2^2\label{eq: split7}\\
    &\leq  \max_b \left\|\frac{\partial\mathcal{L}}{\partial \hat{y}_{b}^t}\frac{\partial^2 \hat{y}^{t}_{b,\ell}}{\partial (z^t)^2}\frac{\partial z^t}{\partial \theta_d^t}(z^t-\tilde{z}^t)\right\|_2^2\\
    &\leq \Lambda_0^2\Lambda_2^2\Lambda_3^2\left(1-\frac{N}{C}\right)B\delta^2,
\end{align}
where $\hat{y}^{t}_{b,\ell}(\tilde{z}^t;\theta_s^t)$ and $\hat{y}^{t}_{b,\ell}(z^t;\theta_s^t)$ denote the logit forwarded with the sparsified intermediate feature $\tilde{z}^t$ and non-sparsified intermediate feature $z^t$, respectively. The first inequality comes from the mean-value theorem and the second inequality is due to \textbf{Assumptions \ref{as: sparsification}} and \textbf{\ref{as: norm boundedness}}. Here, $\Lambda_0$, $\Lambda_2$ and $\Lambda_3$ are spectral norm upper-bound of $\frac{\partial\mathcal{L}}{\partial \hat{y}_{b}^t}$, $\frac{\partial^2 \hat{y}^{t}_{b,\ell}}{\partial (z^t)^2}$ and  $\frac{\partial z^t}{\partial \theta_d^t}$, respectively.

Finally, combining the above inequalities, the error becomes
\begin{align}
    \mathbb{E}\|m^t\|_2^2\leq\Lambda^2\left(1-\frac{N}{C}\right)B\delta^2,
\end{align}
where $\Lambda^2=\Lambda_0^2(\Lambda_1^2+\Lambda_2^2\Lambda_3^2).$

\bibliographystyle{IEEEtran}  
\bibliography{IEEEabrv,reference}
\begin{IEEEbiography}
	[{\includegraphics[width=1in,height=1.25in,clip,keepaspectratio]{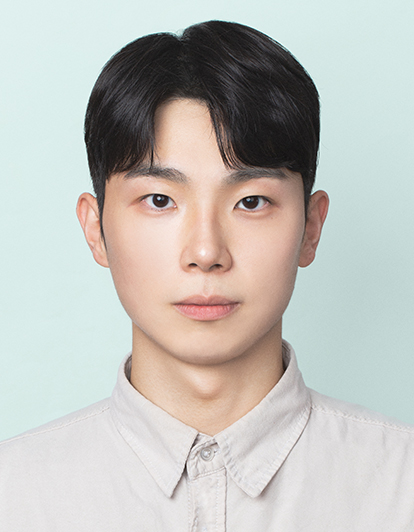}}]{Bumjun Kim}(Graduate Student Member, IEEE) received the B.S. degree from the Department of Electronic Engineering, Jeonbuk National University, Jeonju, South Korea, in 2020. He is currently pursuing the Ph.D. degree with the Department of Electrical and Computer Engineering, Seoul National University, Seoul. His research interests include wireless communications, distributed learning, and semantic communication.
\end{IEEEbiography}

\begin{IEEEbiography}
	[{\includegraphics[width=1in,height=1.25in,clip,keepaspectratio]{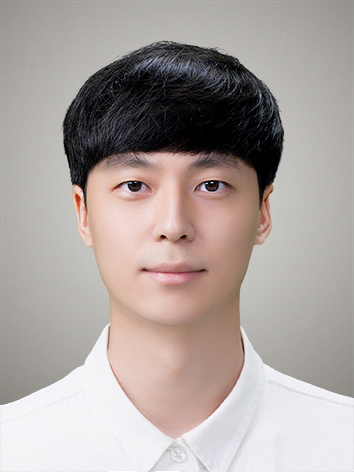}}]{Yoon Huh}(Graduate Student Member, IEEE) received the B.S. degree (summa cum laude) from the School of Electrical and Electronics Engineering, Yonsei University, Seoul, South Korea, in 2022. He is currently pursuing the Ph.D. degree with the Department of Electrical and Computer Engineering, Seoul National University, Seoul, South Korea. His research interests include wireless communication, semantic communication, and deep learning.
\end{IEEEbiography}
\begin{IEEEbiography}
	[{\includegraphics[width=1in,height=1.25in,clip,keepaspectratio]{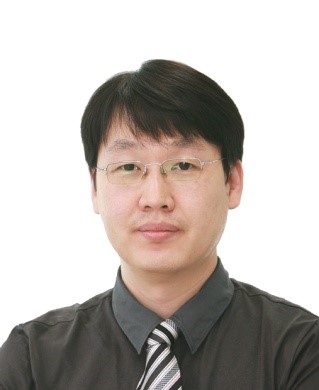}}]{Wan Choi}(Fellow, IEEE) is Professor of Department of Electrical and Computer Engineering, Seoul National University, Seoul, Korea. He is Director of Institute of New Media and Communications (INMC), Seoul National University, from Feb. 2025. From Feb. 2007 to Feb. 2020, he was Professor of School of Electrical Engineering, Korea Advanced Institute of Science and Technology (KAIST), Daejeon, Korea. He received the B.Sc. and M.Sc. degrees from the School of Electrical Engineering and Computer Science (EECS), Seoul National University (SNU), Seoul, Korea, in 1996 and 1998, respectively, and the Ph.D. degree in the Department of Electrical and Computer Engineering at the University of Texas at Austin in 2006. From 1998 to 2003, he was a Senior Member of the Technical Staff of the R\&D Division of KT, Korea, where he researched 3G CDMA systems. 
    
    He is the recipient of IEEE Vehicular Technology Society Jack Neubauer Memorial Award (Best System Paper Award) in 2002. He also received the IEEE Vehicular Technology Society Dan Noble Fellowship Award, the IEEE Communication Society Asia Pacific Young Researcher Award, the NAAI Distinguished Artificial Intelligence Scholar Award, the Okawa Foundation Research Grant Award, the Haedong Research Award and the Haedong Young Scholar Award from KICS, and the Irwin-Jacobs Award from Qualcomm and KICS. While at the University of Texas at Austin, he was the recipient of William S. Livingston Graduate Fellowship and Information and Telecommunication Fellowship from Ministry of Information and Communication (MIC), Korea. He is an Area Editor for the IEEE Transactions on Wireless Communications from Aug. 2022 and an Editor for the IEEE Transactions on Vehicular Technology from Apr. 2011. He is the co-Editor-in-Chief of IEEE/KICS Journal of Communications and Networks from Dec. 2024. He served as the Executive Editor Chair for the IEEE Transactions on Wireless Communications (2019- 2021) and Executive Editor (2014-2019). He was also an Editor for the IEEE Transactions on Wireless Communications (2009-2014), for the IEEE Wireless Communications Letter (2012-2017), and as Guest Editor for the IEEE Journal on Selected Areas in Communications, IEEE Communications Magazine, and IEEE Open Journal of the Communications Society. He is a member of National Academy of Engineering of Korea (NAEK).
\end{IEEEbiography}

\end{document}